\documentclass{article}

\usepackage{microtype}
\usepackage{graphicx}
\usepackage{subfigure}
\usepackage{multirow}
\usepackage{booktabs} 

\usepackage{hyperref}

\usepackage[accepted]{icml2024}

\usepackage{amsmath}
\usepackage{amssymb}
\usepackage{mathtools}
\usepackage{amsthm}

\usepackage[capitalize,noabbrev]{cleveref}

\theoremstyle{plain}
\newtheorem{theorem}{Theorem}[section]

\newtheorem{lemma}[theorem]{Lemma}
\newtheorem{corollary}[theorem]{Corollary}
\theoremstyle{definition}
\newtheorem{definition}[theorem]{Definition}

\theoremstyle{remark}

\DeclareMathOperator{\var}{Var}
\DeclareMathOperator{\Sol}{Sol}

\DeclareMathOperator{\Geomtric}{Geometric}

\DeclareMathOperator{\Rollout}{Rollout}
\DeclareMathOperator{\expect}{\mathbb{E}}
\DeclareMathOperator*{\usexpect}{\mathbb{E}}
\DeclareMathOperator*{\argmax}{arg\,max}

\newcommand{\bs}[1]{\boldsymbol{#1}}
\newcommand{\eps}{\varepsilon}
\newcommand{\entropy}{\mathrm{H}}
\newcommand{\sdmdp}{\rev{DSMDP}}
\newcommand{\expldiff}{J_{\text{\textnormal{explore}}}}
\newcommand{\learndiff}{J_{\text{\textnormal{learn}}}}
\newcommand{\DKL}[2]{D_\mathrm{KL}\left(#1 \parallel #2\right)}
\newcommand{\M}{\mathcal{M}}
\newcommand{\incompress}{\mathrm{IC}}

\newcommand{\rev}[1]{#1}
\newcommand{\revblock}{}

\newcommand{\base}{\rev{0}}
\newcommand{\aug}{\rev{+}}
\newcommand{\unif}{\rev{\text{unif}}}
\newcommand{\solv}{\text{\textnormal{solvable}}}

\usepackage[textsize=tiny]{todonotes}

\icmltitlerunning{When Do Skills Help Reinforcement Learning?}

\begin{document}

\twocolumn[
\icmltitle{When Do Skills Help Reinforcement Learning? \\ A Theoretical Analysis of Temporal Abstractions}





\begin{icmlauthorlist}
\icmlauthor{Zhening Li}{mit}
\icmlauthor{Gabriel Poesia}{stanford}
\icmlauthor{Armando Solar-Lezama}{mit}
\end{icmlauthorlist}

\icmlaffiliation{mit}{MIT CSAIL, Cambridge, MA, USA}
\icmlaffiliation{stanford}{Stanford University, Stanford, CA, USA}

\icmlcorrespondingauthor{Zhening Li}{zli11010@csail.mit.edu}

\icmlkeywords{Machine Learning, ICML}

\vskip 0.3in
]



\printAffiliationsAndNotice{}  

\begin{abstract}
Skills are temporal abstractions that \rev{are} intend\rev{ed}
to improve reinforcement learning (RL) performance through hierarchical RL.
Despite our intuition about the properties of an environment that make skills useful,
a precise characterization has been absent.
We provide the first such characterization, focusing on the utility of deterministic skills
in deterministic sparse\rev{-}reward environments with finite action spaces.
We show theoretically and empirically that RL performance gain from skills
is worse in environments where solutions to states are less compressible.
\rev{Additional} theoretical results suggest that skills benefit exploration more than
they benefit learning from existing experience,
and that using unexpressive skills such as macroactions may worsen RL performance.
We hope our findings can guide research on automatic skill discovery
and help RL practitioners better decide when and how to use skills.
\end{abstract}

\section{Introduction} \label{sec:intro}

In most real-world sequential decision making problems, agents are only given
\emph{sparse rewards} for their actions. This makes reinforcement learning (RL) challenging,
as agents can only recognize good behavior after long sequences of good decisions.
This issue can be \rev{mitigated} by leveraging \textit{temporal abstractions} \citep{sutton1999options},
also known as \textit{skills}.
A skill is a high-level action --- such as a fixed sequence of actions (\textit{macroaction})
or a sub-policy with a termination condition (\textit{option}) --- that is expected
to be useful in a large number of states.
Skills can be hand-engineered to perform subtasks \citep{pedersen2016robot,he2011macro}
or learned from experience
\citep{machado2017laplacianoption,bacon2017optioncritic,barreto2019optionkeyboard,kipf2019compile,jiang2022love,li2022lemma}.
Incorporating skills into the agent's action space (\textit{hierarchical RL}) allows it to act at a higher level
and reach goals in fewer steps, which may improve exploration and thus RL performance.

Despite their appeal, skills have not seen widespread use.
In fact, they were not involved in most major breakthroughs and applications of RL,
such as surpassing human-level performance in all Atari games \citep{badia2020agent57},
RLHF for aligning LLMs with human preferences \citep{ouyang2022rlhf},
AlphaTensor for faster matrix multiplication \citep{fawzi2022alphatensor},
and AlphaDev for faster sorting \citep{mankowitz2023alphadev}.
A reason skills have not been widely adopted is that they sometimes do not improve
RL performance and it is unclear how to determine beforehand whether they would.
While several methods have been developed to automatically discover skills,
most of them require the practitioner to decide whether to use skills at all.
To our knowledge, LEMMA \citep{li2022lemma} is the only algorithm that automatically
decides whether skills are useful by learning the optimal number of skills --- zero would mean that skills do not help.
However, this is accomplished by optimizing a heuristic objective that does not necessarily
reflect the benefits to RL.
Other skill discovery algorithms such as Option-Critic \citep{bacon2017optioncritic},
eigenoptions \citep{machado2017laplacianoption}, deep skill chaining \citep{bagaria2019deepskillchaining},
LOVE \citep{jiang2022love} and COPlanLearn \citep{nayyar2023coplanlearn}
determine the number of skills using a hyperparameter.
A better understanding of how exactly skills benefit RL may guide research
in automatically determining whether skills would be useful in an environment
and the optimal number to learn if they are.
Such an understanding can also provide insight into why skills
do not work in certain environments as well as help practitioners better
decide whether to use skills for a given RL task.


Our work provides a theoretical analysis of when and how skills and hierarchical RL
benefit RL performance in deterministic sparse-reward environments. 
We hope our insights will serve to guide research in automatic skill discovery
including the automatic determination of whether to use skills,
and allow practitioners to better understand the kinds of environments where skills are helpful. In summary, we make the following contributions:

\begin{itemize}
    \item We define two metrics --- $p$-exploration difficulty and $p$-learning difficulty --- that quantify the hardness of exploration and learning from experience
    in a deterministic sparse-reward environment with a finite action space.
    We show empirically that these metrics correlate strongly with the sample complexity of several RL algorithms (\Cref{sec:rldiff}).
    \item We define \rev{two closely related} metric\rev{s} that \rev{measure}
    the incompressibility of solutions to states
    generated by the environment. Under mild assumptions, we prove lower bounds on
    the change in $p$-learning difficulty and $p$-exploration difficulty due to
    deterministic skills in terms of the incompressibility measure\rev{s}.
    We show that skills are better suited to decreasing
    $p$-exploration difficulty rather than $p$-learning difficulty,
    and less expressive skills are less apt at decreasing the difficulty metrics.
    In particular, for each difficulty metric, we demonstrate the existence of environments where
    incorporating macroactions provably increases it (\Cref{sec:learndiff,sec:expldiff}).
    \item We show empirically that \rev{macroactions and deep neural options are less beneficial
    in environments with higher incompressibility} (\Cref{sec:expt_incompress}).
    \item \rev{We describe how to derive skill learning objectives from our incompressibility metrics (\Cref{sec:skill_learning}).}
\end{itemize}

\rev{
All proofs are found in \cref{app:proofs}. Code for experiments are publicly available at
\url{https://github.com/uranium11010/rl-skill-theory}.
}

\section{Preliminary Definitions} \label{sec:def}

We first introduce basic definitions related to deterministic sparse-reward Markov decision
processes (MDPs), which are the focus of this paper.
We choose to focus on sparse-reward environments since skills are purported
to alleviate the sparse-reward problem. Despite our focus on deterministic environments,
a large number of environments both in the standard RL literature
(e.g., the original Atari game environments \citep{bellemare2013atari}
and MuJoCo \citep{todorov2012mujoco})
and in applications of RL
(e.g., program synthesis \citep{ellis2019programsynthesis,mankowitz2023alphadev}
and mathematical reasoning \citep{kaliszyk2018rltheoremproving,poesia2021conpole,wu2021tacticzero})
are deterministic.
Furthermore, by focusing on a special case of MDPs, our hardness results ---
lower bounds on the change in difficulty due to skills
--- suggest that improving RL using skills in the general case of stochastic environments
can be at least as hard.
\rev{Finally, \cref{app:stochastic} provides preliminary results on generalizing to stochastic environments,
suggesting that many insights obtained from studying deterministic environments apply to stochastic ones as well.}

\begin{definition} \label{def:sdmdp}
    A \textit{\rev{deterministic sparse-reward} MDP (\sdmdp)} is defined by a 4-tuple $\M = (S, A, T, g)$ where
    $S$ is the state space, $A$ is the action space, $T: (S\setminus\{g\}) \times A \to S$ is the deterministic transition function and $g \in S$ is the goal state.
\end{definition}
Note that environments that have multiple goal states can also
be formulated as \sdmdp s by merging these goal states into a single goal state.
The \texttt{CompILE2} environment introduced in \cref{sec:expt_diff}
is one such example --- see \cref{app:envs} for more details.

Borrowing terminology commonly used in symbolic reasoning domains,
we say ``solve a state'' as a shorthand for ``finding a sequence of actions that lead to the goal state,''
and we call such a sequence of actions a \textit{solution}. This is formalized below.
\begin{definition}
    A \textit{solution} to a state $s \in S\setminus\{g\}$ of a \sdmdp\ $\M = (S, A, T, g)$
    is a sequence of actions $(a_1, \ldots, a_l) \in A^l$ ($l \geq 1$)
    such that applying the sequence of actions starting in $s$ results in the goal state $g$:
    \begin{equation}
        T(s, (a_1, \ldots, a_l)) = g,
    \end{equation}
    where
        $T(s, (a_1, \ldots, a_l)) := T(\cdots(T(s, a_1), a_2)\cdots, a_l)$
    denotes the result of applying action sequence $(a_1, \ldots, a_l)$ to state $s$.
    Here, $l > 0$ is called the \textit{length} of the solution.
    We will denote by $\Sol_\M(s)$ the set of solutions to $s$
    and $d_\M(s) = \min_{\sigma \in \Sol_\M(s)} |\sigma|$ the length of a shortest solution to $s$.
\end{definition}
Note that a state can have no solutions.
For example, in domains where we'd like to formalize the notion of ``death,''
one could transition to a ``dead state'' that goes to itself for all actions taken,
and that dead state has no solutions.
In contrast, states that have at least one solution are called \textit{solvable} states.

Some results \rev{in this paper} assume that no two states share a solution,
a property we call \textit{solution separability}.
\begin{definition}
    A \sdmdp\ is \textit{solution-separable}
    if no sequence of actions is a solution to more than one state.
\end{definition}
Any \sdmdp\ with invertible transitions is solution-separable.
Here, we say a \sdmdp\ $(S, A, T, g)$ has invertible transitions if
$s = s'$ whenever $T(s, a) = T(s', a)$ and $T(s, a)$ is either solvable or the goal.
Examples include
    (a) all twisty puzzles such as the Rubik's cube;
    (b) grid world domains where taking a vacuous action
    (e.g., walking into a wall or picking up a non-existent object)
    leads to instant death;
    (c) sliding puzzles where taking a vacuous action leads to instant death.

The following definition formalizes RL in the episodic setting as applied to a \sdmdp.
\begin{definition}
    In \textit{reinforcement learning (RL) in the episodic setting},
    an agent interacts with an environment (MDP) in \textit{episodes} to learn a policy $\pi(a \mid s)$
    that optimizes the expected cumulative reward \rev{from one} episode.
    For a \sdmdp, the optimal policy is
    \begin{equation} \label{eq:rl}
        \argmax_\pi \usexpect_{\substack{s_0 \sim p_0 \\ (s_0, a_1, \ldots, a_l, s_l) \sim \Rollout_\pi(s_0)}}
        \left[\gamma^{l-1} \bs{1}[s_l = g]\right].
    \end{equation}
    Here, $p_0$ is the initial state distribution and $0 < \gamma \leq 1$ is the discount factor.
    $\Rollout_\pi(s_0)$ is the result of
    rolling out policy $\pi$ starting in state $s_0$, stopping when either the goal state is reached
    or $H$ actions have been taken, where $H$ is called the \textit{horizon}
    and sometimes considered part of the definition of an MDP.
    Note that when $\gamma = 1$, then \cref{eq:rl} becomes maximizing the probability that the policy solves $s_0 \sim p_0$.
\end{definition}

Now, we introduce \textit{skills}. Whereas skills need not be deterministic in general,
we are studying deterministic environments and \rev{will thus} focus on deterministic skills.
\begin{definition}
    A \textit{deterministic skill} in a \sdmdp\ is a function from states to finite action sequences.
    \rev{In other words, for each state, we specify the sequence of actions to be taken
    if the agent initiates the skill in that state.
    Note that this sequence is allowed to be empty.}
\end{definition}
We will refer to deterministic skills as simply ``skills.''

\rev{T}he prototypical example of an unexpressive class of skills \rev{is} \textit{macroactions}.
\begin{definition} \label{def:macroaction}
    A \textit{macroaction} is a skill that produces \rev{the same} sequence of actions
    \rev{of length greater than 1 regardless of} the state \rev{in which the skill is initiated}.
\end{definition}

Incorporating skills into a \sdmdp\ is called a \textit{skill augmentation},
which is more precisely defined below.
\begin{definition}
    A \sdmdp\ $\M_{\base} = (S, A_{\base}, T_{\base}, g)$
    augmented with \rev{a finite set of} skills $Z$ is the \sdmdp\ $\M_{\aug} = (S, A_{\aug}, T_{\aug}, g)$
    where $A_{\aug} = A_{\base} \cup Z$, $T_{\aug}(s, a) = T_{\base}(s, a)$ for $a \in A_{\base}$,
    and $T_{\aug}(s, a) = T_{\base}(s, a(s))$ for $a \in Z$.\footnote{
    \revblock Technically, $T_{\aug}$ is a partial function as $T_{\aug}(s, z)$
    is undefined if unrolling the skill $z$ reaches the goal state before the unrolling finishes.
    Thus, in this case, the agent is considered \textit{not} to have reached the goal state.
    (However, our HRL implementation in our experiments follows the more common convention
    that the agent is considered successful in this situation.)
    }
    We say $\M_{\aug}$ is the \textit{$A_{\aug}$-skill augmentation} of $\M_{\base}$.
    We call $A_{\base}$ the \textit{base action space} and $A_{\aug}$ the \textit{skill-augmented action space}.
    Furthermore, if $Z \neq \emptyset$ so that $A_{\base}$ is a proper subset of $A_{\aug}$,
    then we say the skill augmentation is \textit{strict}.
    
    For simplicity, when discussing a base environment $\M_{\base}$ and its skill augmentation $\M_{\aug}$,
    we will abuse notation by writing subscripts ``$\aug$'' or ``$\base$''
    in places where they should really be ``$\M_{\aug}$'' or ``$\M_{\base}$'',
    such as $d_{\base}(s)$ \rev{and} $\Sol_{\base}(s)$
    for $d_{\M_{\base}}(s)$ \rev{and} $\Sol_{\M_{\base}}(s)$.
    We allow repetition of skills and skills are also allowed to overlap with base actions.
    In such cases, $Z$ and $A_{\aug}$ should be interpreted as multisets.
\end{definition}


\section{Quantifying RL Difficulty in a \rev{Deterministic Sparse-Reward} Environment} \label{sec:rldiff}

To study how much skills can reduce the difficulty of applying RL to a \sdmdp,
we need to first quantify this difficulty.
Unfortunately, existing MDP difficulty metrics fail to capture RL difficulty in \sdmdp s
since they were not designed to directly estimate sample efficiency or regret,
but instead appear in loose asymptotic performance bounds of RL algorithms
(see \cref{app:rldiff} for a brief survey).
As a result, they correlate poorly with actual performance measures like total regret \citep{conserva2022hardness}.
We therefore aim to develop difficulty metrics for \sdmdp s by
directly estimating an RL performance measure --- in our case, sample efficiency
--- and to verify them empirically.

Below, we introduce two metrics quantifying the difficulty of applying RL to a
deterministic sparse\rev{-}reward environment,
assuming that the environments compared have the same state space
(e.g., they are different skill augmentations of the same base environment).
We motivate these metrics using heuristic arguments that estimate the sample efficiency
of an RL agent in the episodic setting without assuming any particular RL algorithm.
We then experimentally test how well the metrics correlate with the sample efficiency
of 4 popular RL algorithms in 32 macroaction augmentations of each of 4 base environments.


\subsection{Quantifying Difficulty in Learning from Experience} \label{sec:learndiff_def}

To quantify the complexity of learning a \sdmdp\ from existing experience,
suppose that the agent has gathered enough experience to effectively reduce the remaining
learning problem to a planning problem.
Then \Cref{lem:learndiff} shows that the number of iterations through the entire state space needed to learn
the value of a state is linear in the minimum length of a solution to that state.
\begin{lemma} \label{lem:learndiff}
    Suppose we apply value iteration with discount rate $\gamma=1$ and learning rate $\alpha$
    to a \sdmdp\ $\M = (S, A, T, g)$ with a finite action space.
    In particular, we initialize $V(s) \gets 0$ for $s \neq g$ and $V(g) \gets 1$,
    and at time $t$, we update the entire table using
    \begin{equation} \label{eq:value_iteration}
        V(s) \gets (1 - \alpha) V(s) + \alpha \max_a V(T(s, a)) \quad \text{for all $s \neq g$}.
    \end{equation}
    If $\alpha = 1$, then the number of time steps until the value of a solvable state $s$
    becomes its true value (i.e., $1$) is $d_\M(s)$.
    If $\alpha < 1$, then the number of time steps until the value of a solvable state $s$ is within
    $\eps$ of its true value (i.e., $1 - V(s) < \eps$) is
    \[
        \Theta\left(\frac{d_\M(s) + \log(1/\eps)}{\alpha}\right).
    \]
\end{lemma}
Since each iteration has a complexity of $\Theta(|S||A|)$, the total complexity for learning
the value of a state $s$ is $\Theta(|S||A|d_\M(s))$ for constant $\alpha, \eps$.
If we apply the same intuition to the RL setting,
then we would expect that learning the optimal policy at a state $s$ requires $\Theta(d_\M(s))$
``iterations,'' where one ``iteration'' involves the agent sampling experiences
that effectively cover the entire space of state-action pairs.
Thus, as a rough estimation, approximately $\Theta(|S_\text{eff}||A|d_\M(s))$ samples are
needed to learn the policy at state $s$.
Here, $|S_\text{eff}|$ is some effective size of the state space,
counting only those states that we ``care about,'' i.e., those with positive $p_0(s)$
or that are part of (short) solutions to states with positive $p_0(s)$.
For constant $|S_\text{eff}|$, this estimation of the sample complexity
motivates using a weighted average of $|A|d_\M(s)$ over states $s$ to measure
the complexity of learning from experience.
\begin{definition}
    Let $\M = (S, A, T, g)$ be a \sdmdp\ with finite action space $A$.
    For a probability distribution $p$ on solvable states,
    the \textit{$p$-learning difficulty} of $\M$ is defined as
    \begin{equation} \label{eq:learndiff}
        \learndiff(\M; p) = |A| \mathbb E_{s \sim p}[d_\M(s)]
    \end{equation}
    where $d_\M(s)$ is the length of a shortest solution to $s$.
\end{definition}
The distribution $p$ assigns higher importance to states that we care more about learning to solve.
If $p_0$ denotes the initial state distribution of the MDP, then $p$ should be higher for states with higher $p_0$.
For simplicity, we can just take $p$ to be $p_0$.

The $p$-learning difficulty can be viewed as a generalization of diameter \citep{auer2008diameter}.
While the diameter of an MDP is originally defined for the continuous learning setting,
a natural extension to the episodic setting for a \sdmdp\ is the maximum length of a solution
to a state, $\max_{s \neq g} d_\M(s)$. Ignoring the $|A|$ factor,
this is the $p$-learning difficulty when $p$ is zero for all but the state(s) with the largest $d_\M(s)$.

\subsection{Quantifying Difficulty in Exploration} \label{sec:expldiff_def}

$p$-learning difficulty does not take into account the complexity of gathering the needed experience:
learning a state $s$ starts to take place only after the agent has seen state-action pairs that form a chain
leading from $s$ to the goal state.
Thus, as a simplification, an agent's learning process in the episodic setting can be roughly divided into two stages:
the first stage is dominated by exploration, where the agent tries to find reward signal and gather experience;
the second stage is dominated by learning, where the agent learns from the experience.
The sample efficiency of the learning stage is captured by the $p$-learning difficulty.
Let us now motivate the definition of \textit{$p$-exploration difficulty}
by estimating the sample efficiency of the exploration stage.

Suppose that the initial exploration policy is a uniformly random policy,
and let $q(s)$ denote the probability that such a policy solves $s$ in one episode.
Assuming that the policy remains roughly uniform until the agent finally solves $s$ for the first time,
the expected number of episodes until this happens is $1/q(s)$,
and the number of environment steps taken is $H/q(s)$ where $H$ is the horizon.
To obtain an upper bound on the expected total number of steps taken to find a solution to every state,
we simply sum this expression over all states to arrive at $N_\text{sum} = H \sum_s \frac{1}{q(s)}$.
Note that this can be a significant overestimate of the true sample complexity:
solving a state $s$ often updates the agent in a way that helps it solve states whose solutions contain $s$.
We will address this issue later.

For a constant horizon $H$ and state space size,
$N_\text{sum} \propto \mathbb E_{s \sim p}\left[1/q(s)\right]$ where $p$ is a uniform distribution over all states.
As with the $p$-learning difficulty, we generalize this to allow different weights $p(s)$ to be assigned to different states.
For example, if a state has small $q(s)$ but the MDP's initial \rev{state} distribution $p_0$ assigns almost zero probability
to $s$, then we \rev{can} afford \rev{not to learn to solve} $s$ and this can be reflected by having \rev{$p(s) \approx 0$}.
For simplicity, \rev{we can simply set} $p$ \rev{to} $p_0$, as with the $p$-learning difficulty.

We now address the issue of overestimating the sample \rev{complexity}.
In practice, this overestimation is more significant when $q(s)$ for different $s$ are more disparate.
In \sdmdp s where states vary in difficulty (vary in $q(s)$),
solving easy states (states with large $q(s)$) generally updates the agent in a way
that helps it find solutions to harder states (states with small $q(s)$).
For this reason, we find empirically (\cref{app:expt_results_AM}) that
the arithmetic mean $N_\text{AM} = \mathbb E_{s \sim p}[1/q(s)]$
is outperformed by the geometric mean $N_\text{GM} = \exp(\mathbb E_{s \sim p}[\log(1/q(s))])$,
which is lower than $N_\text{AM}$ when there's variety in $1/q(s)$.
Although this estimation of exploration sample complexity is quite rough,
it is difficult to make better estimates without knowing details of the MDP structure and RL algorithm.
Also, the resultant definition of $p$-exploration difficulty already performs well empirically
on several environments for several RL algorithms (\cref{sec:expt_diff}).

Finally, we take the logarithm of $N_\text{GM}$ as that simplifies notation in our theoretical results.
We also replace the fixed horizon with a random horizon sampled from a \rev{geometric} distribution
to simplify theoretical analysis.
\begin{definition}
    Let $\M = (S, A, T, g)$ be a \sdmdp\ with finite action space $A$.
    For a probability distribution $p$ on solvable states and $0 \leq \delta < 1$,
    the \textit{$\delta$-discounted $p$-exploration difficulty} of $\M$ is defined as
    \begin{equation} \label{eq:expldiff}
        \expldiff(\M; p, \delta) = \mathbb E_{s \sim p}[-\log q_{\M,\delta}(s)]
    \end{equation}
    where
    \begin{equation} \label{eq:q}
        q_{\M,\delta}(s) := \sum_{\sigma \in \Sol(s)} \left(\frac{1 - \delta}{|A|}\right)^{|\sigma|}
    \end{equation}
    is the probability that the following policy solves $s$: at every time step, terminate with probability $\delta$
    and \rev{choose an action uniformly at random} with probability $1 - \delta$.
    $q_{\M,\delta}(s)$ is also the probability that the uniform\rev{ly} random policy solves $s$ within
    a horizon of length $H$, where $H + 1$ is sampled from the geometric distribution with parameter $\delta$.
\end{definition}

\subsection{Experiments} \label{sec:expt_diff}

\begin{table*}[t]
\caption{Correlations between $\log N$ and $\log J$
where $N$ is the number of environment steps the agent takes
to learn the environment
and $J = \lambda\learndiff + (1 - \lambda)\exp(\expldiff)$ is a weighted average of
the $p$-learning difficulty and the exponential of the $p$-exploration difficulty.
Convergence criteria include reaching a certain reward threshold $r^*$
(0.5 for \texttt{RubiksCube222} and 0.9 for the other environments)
or reaching a certain threshold $\Delta Q^*$ or $\Delta V^*$ in the $p$-weighted average error in action or state values
(0.2 for \texttt{RubiksCube222} and 0.05 for the other environments).
The value of $\lambda \in [0, 1]$ was chosen so that the correlation was maximized.
Data points where the algorithm never converges before the experiment run ends (100M environment steps)
were excluded from the calculation of the correlation.
The reported errors are standard errors of the mean over \rev{5} random seeds.
}
\label{tab:expt_results}
\vskip 0.15in
\begin{center}
\begin{small}
\begin{tabular}{rrcccc}
\toprule
&& $\log J_\mathtt{CliffWalking}$ & $\log J_\mathtt{CompILE2}$ & $\log J_\mathtt{8Puzzle}$ & $\log J_\mathtt{RubiksCube222}$ \\
\toprule
\multirow{2}{*}{Q-Learning}
& $\log N_{r \geq r^*}$
    & \rev{0.947 $\pm$ 0.006} & \rev{0.792 $\pm$ 0.025} & \rev{0.403 $\pm$ 0.036} & \rev{0.857 $\pm$ 0.023} \\
& $\log N_{\overline{\Delta Q} \leq \Delta Q^*}$
    & \rev{0.953 $\pm$ 0.008} & \rev{0.786 $\pm$ 0.023} & \rev{0.671 $\pm$ 0.056} & \rev{0.937 $\pm$ 0.003} \\
\midrule
\multirow{2}{*}{Value iteration}
& $\log N_{r \geq r^*}$
    & \rev{0.933 $\pm$ 0.009} & \rev{0.825 $\pm$ 0.018} & \rev{0.693 $\pm$ 0.051} & \rev{0.785 $\pm$ 0.031} \\
& $\log N_{\overline{\Delta V} \leq \Delta V^*}$
    & \rev{0.951 $\pm$ 0.015} & \rev{0.849 $\pm$ 0.013} & \rev{0.885 $\pm$ 0.011} & \rev{0.748 $\pm$ 0.029} \\
\midrule
\multirow{1}{*}{REINFORCE}
& $\log N_{r \geq r^*}$
    & \rev{0.949 $\pm$ 0.006} & \rev{0.869 $\pm$ 0.013} & \rev{0.678 $\pm$ 0.020} & \rev{0.892 $\pm$ 0.029} \\
\midrule
\multirow{1}{*}{DQN}
& $\log N_{r \geq r^*}$
    & \rev{0.789 $\pm$ 0.028} & \rev{0.758 $\pm$ 0.076} & \rev{0.583 $\pm$ 0.039} & \rev{0.753 $\pm$ 0.019} \\
\bottomrule
\end{tabular}
\end{small}
\end{center}
\vskip -0.1in
\end{table*}

In motivating $p$-learning difficulty and $p$-exploration difficulty,
we made significant approximations to estimate the sample complexity without assuming a particular environment or RL algorithm.
Despite this, we show empirically that a combination of the two difficult metrics predicts sample complexity well
across a variety of environments and RL algorithms.

We study four deterministic sparse-reward environments:
    (a) \texttt{CliffWalking}, a simple grid world \citep{sutton2018reinforcement};
    (b) \texttt{CompILE2}, the CompILE grid world with visit length 2 \citep{kipf2019compile};
    (c) \texttt{8Puzzle}, the 8-puzzle;
    (d) \texttt{RubiksCube222}, the 2x2 Rubik's cube.
For the computation of $p$-learning difficulty and $p$-exploration difficulty to be feasible,
$p$ needs to have finite support over a sufficiently small number of states ($\sim 10^7$ or less). 
\rev{To mitigate this limitation, we chose environments for which
there exist larger versions with a similar MDP structure.
For example, the 2x2 Rubik's cube should behave similarly to the 3x3 cube, 4x4 cube, etc.,
and the 8-puzzle should behave similarly to the 15-puzzle, 24-puzzle, etc.}

Each environment has 32 action space variants, with one being the base environment
(the trivial skill augmentation) and 31 with different sets of macroactions.
One macroaction augmentation is calculated using LEMMA \citep{li2022lemma} on offline data derived from breadth-first search;
5 are variations of that macroaction augmentation; and 25 are generated randomly.
More details are given in \cref{app:envs}.

We evaluate how well a combination of $p$-learning difficulty and $p$-exploration difficulty
captures the sample complexity of 4 RL algorithms on the different variants of each environment.
The algorithms are:
    (a) Q-learning \citep{watkins1989qlearning};
    (b) Value iteration \citep{bellman1957valueiteration}, modified to the RL setting, similar to \citep{agostinelli2019cube};
    (c) REINFORCE \citep{williams1992reinforce}, made tabular by parameterizing the policy directly with the logits of the actions;
    (d) Deep Q-networks (DQN) \citep{mnih2015dqn}.

According to \cref{sec:learndiff_def,sec:expldiff_def},
we expect $\learndiff$ to scale roughly linearly with the sample complexity of learning from experience
and $\exp(\expldiff)$ to scale roughly linearly with the sample complexity of exploration.
We thus choose a \rev{weighted average} $J = \lambda\learndiff + \rev{(1 - \lambda)}\exp(\expldiff)$
($0 \leq \lambda \leq 1$) to represent the combined difficulty.
The discount $\delta$ used in the $p$-exploration difficulty
is set to $1/H$, where $H$ is the environment's horizon.
The sample complexity $N$ and the combined difficulty $J$
spanned several orders of magnitude in \texttt{CliffWalking} and \texttt{CompILE2},
so we took the logarithm of both before computing their Pearson correlation coefficient.
The value of $\lambda$ was chosen to maximize this correlation.
The results are summarized in \cref{tab:expt_results}.
Most correlation values are \rev{at least around} 0.7,
demonstrating that combining $p$-learning difficulty and $p$-exploration difficulty
allows us to capture a significant portion of the variation in RL sample efficiency
on different action space variants of the same environment.

We also conducted experiments to directly test \cref{lem:learndiff}
by computing the correlation between the number of iterations
it takes value iteration to converge and the $p$-weighted average solution length
(\cref{app:expt_learndiff}).
In addition to state value iteration, we also considered Q-value iteration
to simulate Q-learning. With \rev{two} exception\rev{s}, all correlations are above 0.9,
thus empirically corroborating \cref{lem:learndiff}.

\section{Effect of Skills on Learning from Experience} \label{sec:learndiff}

Part of our goal is to understand what makes a particular set of skills helpful for an RL agent.
One intuition articulated in prior work \cite{jiang2022love,kipf2019compile} is that
skills help \emph{compress} optimal trajectories,
making them shorter and thus more likely to be found during exploration.
But, conversely, data distributions can be provably \emph{incompressible} when
their entropy is too high \cite{cover1994information}.
As a result, we expect that skills are less likely to be helpful when
the distribution of optimal trajectories in the environment is incompressible.
{\revblock
This intuition is made precise by \Cref{thm:learndiff},
which states that the ratio between the new and old $p$-learning difficulties
after an $A_{\aug}$-skill augmentation is lower-bounded by the product of an incompressibility measure
and a factor penalizing large $|A_{\aug}|$.
Before stating the theorem, let's first define this incompressibility measure.
\begin{definition}
    Let $\M_{\base} = (S, A_{\base}, T_{\base}, g)$ be a \sdmdp\ with finite $|A_{\base}| > 1$
    and $\M_{\aug} = (S, A_{\aug}, T_{\aug}, g)$ its $A_{\aug}$-skill augmentation.
    Let $p$ be a distribution over solvable states.
    The \emph{$A_{\aug}$-merged $p$-incompressibility} is defined as
    \begin{equation} \label{eq:incomp_A}
        \incompress_{A_{\aug}}(\M_{\base}; p) = \sup_{0 < \eps < 1} \frac{\entropy[P_{\aug}] - \log\left(\frac{1 - \eps}{\eps}\right)}{\mathbb E_{s \sim p}[d_{\base}(s)] \log\left(\frac{|A_{\base}|}{1 - \eps}\right)}.
    \end{equation}
    Here, $P_{\aug}$ is the distribution of canonical shortest solutions in $\M_{\aug}$
    to states sampled from $p$, where
    the canonical shortest solutions are chosen such that $\entropy[P_{\aug}]$ is maximized.
    Note that $\entropy[P_{\aug}]$ is the entropy of the state distribution after
    states with the same canonical solution in $\M_{\aug}$ have been merged into one state.
    Thus, it has the property $\entropy[P_{\aug}] \leq \entropy[p]$,
    where equality holds iff all states in the support of $p$ have different canonical solutions.
\end{definition}
$A_{\aug}$-merged $p$-incompressibility can be understood as the coding efficiency of
using base actions to write solutions to states sampled from $p$
as opposed to using a code optimized for the distribution of shortest solutions with skills.
More precisely, we can write
\begin{equation}
    \incompress_{A_{\aug}}(\M_{\base}; p) = \sup_{0 < \eps < 1} \frac{\entropy[P_{\aug}] - \log\left(\frac{1 - \eps}{\eps}\right)}{\entropy[P_{\base}, P_{\base,\unif,\eps}] - \log\left(\frac{1 - \eps}{\eps}\right)},
\end{equation}
where
    $\entropy[P_{\aug}]$ is the optimal expected number of bits needed to encode a (canonical) shortest solution
    in $\M_{\aug}$ to a state $s \sim p$, and
    $\entropy[P_{\base}, P_{\base,\unif,\eps}]$ denotes the cross entropy between $P_{\base}$ and $P_{\base,\unif,\eps}$. 
    $P_{\base}$ is the distribution
    of shortest solutions to states sampled from $p$ containing only base actions.
    $P_{\base,\unif,\eps}(\sigma) = \eps(1 - \eps)^{|\sigma| - 1} |A_{\base}|^{-|\sigma|}$ is a uniform prior over base action sequences.
    $\entropy[P_{\base}, P_{\base,\unif,\eps}]$ is thus the expected number of bits required to encode a shortest solution
    using a fixed-length code over base actions $A_{\base}$,
    optimized for a termination symbol that appears at the end of each time step with probability $\eps$.

We now introduce the theorem, which shows how
$A_{\aug}$-merged $p$-incompressibility can be used to bound how much skills in $A_{\aug}$
can improve $p$-learning difficulty.
\begin{theorem} \label{thm:learndiff}
    Let $\M_{\aug} = (S, A_{\aug}, T_{\aug}, g)$ be the $A_{\aug}$-skill augmentation of the \sdmdp\ $\M_{\base} = (S, A_{\base}, T_{\base}, g)$ with finite $|A_{\base}| > 1$,
    and $p$ a probability distribution over solvable states.
    Then
    \begin{equation} \label{eq:learndiff_icA}
        \frac{\learndiff(\M_{\aug}; p)}{\learndiff(\M_{\base}; p)}
        \geq \frac{|A_{\aug}|\log|A_{\base}|}{|A_{\base}|\log|A_{\aug}|}\incompress_{A_{\aug}}(\M_{\base}; p).
    \end{equation}
\end{theorem}
}

\rev{We can use \Cref{thm:learndiff} to understand the effect that the expressivity of skills has
on their ability to improve $p$-learning difficulty.}\footnote{
\rev{See \cref{app:express} for a more formal treatment where the incompressibility measure in \Cref{thm:learndiff}
is replaced with one defined explicitly in terms of a quantitative measure of expressivity.
}
}
More expressive skills can encode more diverse behavior and thus allow a larger number of action sequences
to be encoded as the same skill. This allows states to share solutions more often, which decreases $\entropy[P_{\aug}]$
and hence $\incompress_{A_{\aug}}(\M_{\base}; p)$.
As a result, the lower bound on the $p$-learning difficulty ratio decreases.
As concrete examples, if we place no restriction on what kinds of skills are allowed,
then \rev{we can simply include a single skill that solves all solvable states,
resulting in} $\incompress_{A_{\aug}}(\M_{\base}; p) = 0$ and $\learndiff(\M_{\aug}; p) = |A_{\base}| + 1$.
This is less than $\learndiff(\M_{\base}; p)$ whenever $\mathbb E_{s \sim p}[d_{\base}(s)] > 1 + 1/|A_{\base}|$,
which is true for all RL environments of practical interest.
If a skill is allowed to be a concrete sequence of actions and loops of actions,
then states whose solutions involve different numbers of repetitions of the same component
will have the same solution containing a skill with a loop whose body is that component.
Thus, $\entropy[P_{\aug}] < \entropy[p]$ but is larger than the value of zero obtained when no restriction is placed on skills.
Finally, if skills are restricted to macroactions, then distinct solutions remain distinct
after rewriting with macroactions, and so the $A_{\aug}$-merged $p$-incompressibility
achieves its maximum value.
\rev{In solution-separable environments, this maximum value
is equal to the \textit{unmerged $p$-incompressibility} (\cref{def:incomp_unmerged}),
in which case \cref{thm:learndiff} can be restated in terms of it
(\cref{cor:learndiff_ic}).}
\rev{
\begin{definition} \label{def:incomp_unmerged}
    Let $\M = (S, A, T, g)$ be a \sdmdp\ with finite $|A| > 1$ and $p$ a distribution over solvable states.
    The \emph{unmerged $p$-incompressibility} is defined as
    \begin{equation} \label{eq:incomp_unmerged}
        \incompress(\M; p) = \sup_{0 < \eps < 1} \incompress(\M; p, \eps),
    \end{equation}
    where the \emph{$\eps$-discounted unmerged $p$-incompressibility}
    \begin{equation} \label{eq:incomp_discounted} \allowdisplaybreaks
        \incompress(\M; p, \eps) = \frac{\entropy[p] - \log\left(\frac{1 - \eps}{\eps}\right)}{\mathbb E_{s \sim p}[d_\M(s)] \log\left(\frac{|A|}{1 - \eps}\right)}.
    \end{equation}
    It measures incompressibility on a scale from 0 to 1 if $\M$ is solution-separable.
    Furthermore, unlike the $A_{\aug}$-merged $p$-incompressibility,
    it is a function of only $\M$ and $p$ and is thus a general measure
    of the incompressibility of $\M$.
\end{definition}
}
\begin{corollary}[Corollary to \cref{thm:learndiff}] \label{cor:learndiff_ic}
    \rev{In the setup to \cref{thm:learndiff}, suppose $\M_{\base}$ is solution-separable\footnote{
    \rev{See \cref{app:learndiff_ic_gen} for the version of this corollary that does not assume
    solution-separability.}
    }
    and $A_{\aug}$ is a macroaction augmentation.}
    Then
    \begin{equation} \label{eq:learndiff_ic}
        \frac{\learndiff(\M_{\aug}; p)}{\learndiff(\M_{\base}; p)}
        \geq \frac{|A_{\aug}|\log|A_{\base}|}{|A_{\base}|\log|A_{\aug}|}\incompress(\M_{\base}; p).
    \end{equation}
    
\end{corollary}
A direct \rev{consequence} of \rev{the above} corollary is that there exist environments
where incorporating macroactions will always worsen $p$-learning difficulty,
no matter how many there are or what they are.
\begin{corollary}[Corollary to \cref{cor:learndiff_ic}] \label{cor:learndiff}
    \rev{In the setup to \cref{thm:learndiff}, suppose $\M_{\base}$ is solution-separable
    and $A_{\aug}$ is a strict macroaction augmentation.}
    If
    \[
        1 - \incompress(\M_{\base}; p)
        \leq \frac{1}{|A_{\base}| + 1}\left(1 - \frac{1}{\ln|A_{\base}|}\right),
    \]
    then $\learndiff(\M_{\aug}; p) > \learndiff(\M_{\base}; p)$.
\end{corollary}

\section{Effect of Skills on Exploration} \label{sec:expldiff}

To study the properties of a \sdmdp\ that make exploration difficult,
we \rev{have derived} a tight lower bound on the $p$-exploration difficulty of a \sdmdp\ in terms of
the entropy of $p$ and a term representing how dense solutions to states are in the space of all solutions
(\rev{\Cref{thm:expldiff_lb}).}
{\revblock
\begin{definition} \label{def:density}
    Let $\M = (S, A, T, g)$ be a \sdmdp\ with finite action space $A$.
    For $0 \leq \delta < 1$,
    the \emph{$\delta$-discounted solution density} of $\M$ is defined as
    \begin{equation}
        D(\M; \delta) = \sum_s \rho_{\M,\delta}(s),
    \end{equation}
    where
    \begin{align}
        \rho_{\M,\delta}(s) &= \frac{\delta}{1 - \delta} q_{\M,\delta}(s) \notag \\
        &= \sum_{\sigma \in \Sol_{\M}(s)} \delta(1 - \delta)^{|\sigma|-1} |A|^{-|\sigma|}
    \end{align}
    is the probability that a uniformly random action sequence with length sampled from $\Geomtric(\delta)$ solves $s$.
\end{definition}
}
\begin{theorem} \label{thm:expldiff_lb}
    Let $\M_{\aug} = (S, A_{\aug}, T_{\aug}, g)$ be the $A_{\aug}$-skill augmentation of the \sdmdp\ $\M_{\base} = (S, A_{\base}, T_{\base}, g)$ with a finite action space,
    and $p$ a probability distribution over solvable states.
    Then for $0 < \delta < 1$,
    \begin{equation} \label{eq:expldiff_lb}
        \expldiff(\M_{\aug}; p, \delta) \geq
        \entropy[p] - \log\left(\frac{1 - \delta}{\delta} D(\M_{\aug}; \delta)\right).
    \end{equation}
    Furthermore, if the state space is finite and $\delta > \max_s p(s)$,
    then for any $\eps > 0$, there exists an $A_{\aug}$-skill augmentation $\M_{\aug}$ of $\M_{\base}$ such that
    \begin{equation}
        \expldiff(\M_{\aug}; p, \delta) < 
        \entropy[p] - \log\left(\frac{1 - \delta}{\delta} D(\M_{\aug}; \delta)\right) + \eps,
    \end{equation}
    thus showing that the lower bound given above is tight for all finite \sdmdp s and
    a large range of $\delta$.
\end{theorem}
The fact that the lower bound grows with $\entropy[p]$ is intuitive:
\rev{when there are many states that we care about learning to solve ($\entropy[p]$ is large),
it is hard for the agent to gather the experience needed to learn to solve all these states ($\expldiff$ is large)}.
However, incorporating skills only changes the action space and cannot affect $\entropy[p]$.
Skills thus improve exploration by increasing the $\delta$-discounted solution density,
which is interpreted as the density of solutions to states within the space of all action sequences.
Action sequences of length $l$ equally divide a total density of $\delta(1 - \delta)^{l-1}$,
so that the combined density of all possible action sequences is 1.
If $\M_{\aug}$ is solution-separable, then $\sum_s \rho_{\aug,\delta}\rev{(s)} \leq 1$,
whereas if every action sequence solves some state, then $\sum_s \rho_{\aug,\delta}\rev{(s)} \geq 1$.
Skills improve exploration by increasing this density, similar to how \rev{skills reduce}
$A_{\aug}$-merged $p$-incompressibility \rev{by allowing more states to share solutions}.
More expressive skills are more apt at increasing \rev{solution} density. For example,
introducing macroactions in a solution-separable environment results in a solution-separable environment,
so the density remains at most 1.
If we introduce the logic of loops, then states whose solutions involve different repetitions of the same component
can be solved by the same action sequence containing a loop skill, hence increasing \rev{the} density.
In the extreme case where no restriction is placed on the kind of \rev{skills allowed},
we can introduce \rev{many skills, each of which automatically solves all solvable states}.
\rev{The resultant density is approximately $\delta|S_{\rev{\text{solvable}}}|$,
which is usually much larger than 1.}

As a corollary to \cref{thm:expldiff_lb}, increase in $p$-exploration difficulty due to macroactions
is lower-bounded by the $\delta$-discounted unmerged $p$-incompressibility (\cref{eq:incomp_discounted})
in solution-separable environments,
thus providing the $p$-exploration difficulty counterpart to \cref{cor:learndiff_ic}.
\begin{corollary}[Corollary to \cref{thm:expldiff_lb}] \label{cor:expldiff_ic}
    \rev{In the setup to \cref{thm:expldiff_lb}, suppose $\M_{\base}$ is solution-separable, $|A_{\base}| > 1$,
    and $A_{\aug}$ is a macroaction augmentation.}
    Then
    \begin{equation}
        \frac{\expldiff(\M_{\aug}; p, \delta)}{\expldiff(\M_{\base}; p, \delta)} \geq \incompress(\M_{\base}; p, \delta)\rev{.}
    \end{equation}
\end{corollary}
Compared to \cref{cor:learndiff_ic}, the factor $\frac{|A_{\aug}|\log|A_{\base}|}{|A_{\base}|\log|A_{\aug}|}$
penalizing large $A_{\aug}$ is absent, \rev{and the $\sup$ in
$\incompress(\M_{\base}; p) = \sup_{0 < \delta < 1} \incompress(\M_{\base}; p, \delta)$ has been removed.
The resultant weaker bound suggests} that skills are better
suited to improving exploration than learning from experience.
\rev{This is made more precise in \cref{thm:incomp_expl,cor:incomp_expl} below,
but before stating these results, we shall first} give an intuitive explanation for why this is the case.

In discussing the effects of skills on learning from existing experience,
there was a tradeoff between action space size and reducing solution lengths.
Intuitively, while skills allow reward information to propagate to states faster,
a large action space means a larger number of experiences
to iterate through to efficiently cover the space of all state-action pairs $(s, a)$.
Such a tradeoff is not so clear in the effects of skills on exploration.
To improve exploration, skills are chosen so that a uniformly random policy in the augmented action space
is more likely to reach the goal. If skills are expressive enough, this should always be possible,
unless the base action space is already close to optimal.
Of course, the most general skills trivially improve $p$-exploration difficulty
by simply mapping every \rev{solvable} state to the goal, which gives $\expldiff \approx 0$.
But there can be skills that achieve the maximum possible $A_{\aug}$-\rev{merged} $p$-incompressibility
(which appears in the lower bound for $p$-learning difficulty increase \rev{in} \cref{thm:learndiff})
but still decrease $p$-exploration difficulty.
This is \rev{made precise by} the following theorem.
\begin{theorem} \label{thm:incomp_expl}
    Let $\M_{\base} = (S, A_{\base}, T_{\base}, g)$ be a solution-separable \sdmdp\ with finite $|A_{\base}| > 1$
    as well as finite $|S|$.
    Let $p$ be a probability distribution over solvable states.
    For all $\delta > \max_s p(s)$ for which $p \not\equiv \rho_{\base,\delta}$,
    there exists an $A_{\aug}$-skill augmentation $\M_{\aug}$ of $\M_{\base}$ such that:
    \begin{itemize}
        \item There exist distinct shortest solutions in $A_{\aug}$ to all states in the support of $p$
        (namely, $\entropy[P_{\aug}]$ achieves its maximum possible value $\entropy[p]$
        and \rev{thus} $\incompress_{A_{\aug}}(\M_{\base}; p)$ achieves its maximum possible value $\incompress(\M_{\base}; p)$);
        \item $\expldiff(\M_{\aug}; p, \delta) < \expldiff(\M_{\base}; p, \delta)$.
    \end{itemize}
\end{theorem}
\begin{corollary}[Corollary to \cref{thm:incomp_expl}] \label{cor:incomp_expl}
    Assume the setup to \cref{thm:incomp_expl}. If
    \[
        1 - \incompress(\M_{\base}; p)
        \leq \frac{1}{|A_{\base}| + 1}\left(1 - \frac{1}{\ln|A_{\base}|}\right),
    \]
    then there exists a skill augmentation $\M_{\aug}$ of $\M_{\base}$ such that
    \(
        \learndiff(\M_{\aug}; p) > \learndiff(\M_{\base}; p)
    \)
    but
    \(
        \expldiff(\M_{\aug}; p, \delta) < \expldiff(\M_{\base}; p, \delta).
    \)
\end{corollary}
\Cref{cor:incomp_expl} shows that there are environments where 
\rev{skills can} benefit exploration but harm learning from experience.
This again suggests that skills are more apt at improving exploration than learning.

\begin{figure*}[t]
\vskip -0.05in
\begin{center}
\centerline{\includegraphics[width=\textwidth, trim={0.25in 0.05in 0.75in 0.4in}, clip]{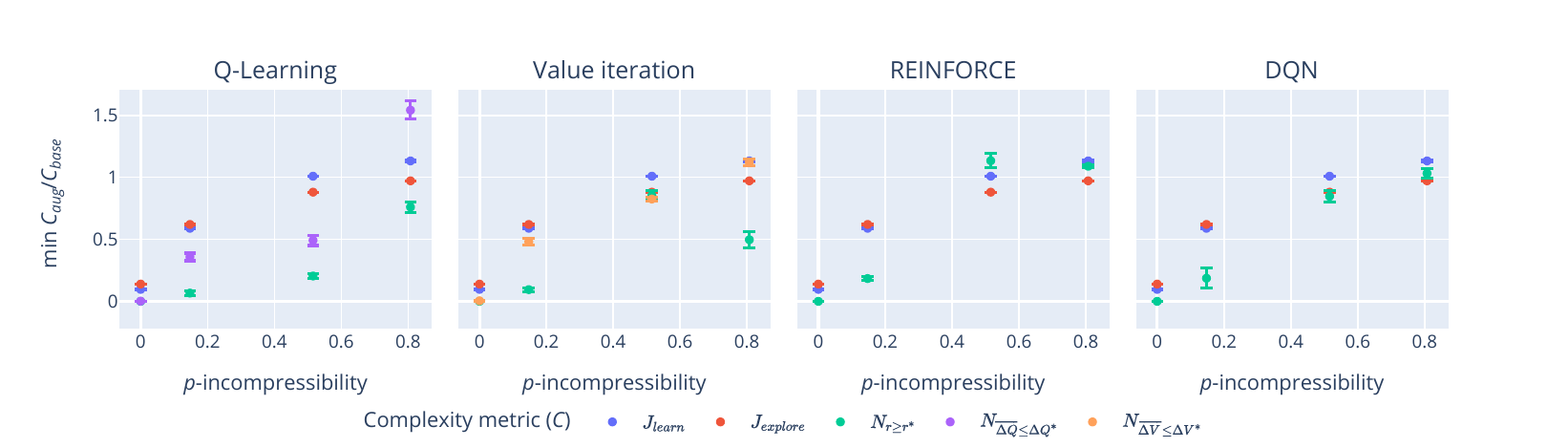}}
\vskip -0.1in
\caption{
\rev{For each of the 4 environments studied, we plot the point $(x, y)$ where
$x$ is the unmerged $p$-incompressibility of the base environment and
$y$ is the best complexity improvement ratio $\min C_{\aug} / C_{\base}$ over
the 31 macroaction augmentations of the base environment.}
Different colors represent different measures $C$ of complexity,
\rev{and different panels correspond to sample complexities $N$ of different RL algorithms.}
The plots corresponding to $p$-learning difficulty ($\learndiff$)
and $p$-exploration difficulty ($\expldiff$) have been repeated across
panels for clearer comparison with the plots corresponding to the sample complexities ($N$) of the RL algorithms.
}
\label{fig:rl_improv-incompress}
\end{center}
\vskip -0.3in
\end{figure*}

As a final discussion on the effect that skills have on exploration, we answer the question:
\rev{are there environments where unexpressive skills like macroactions} always harm exploration?
Unlike \cref{cor:learndiff_ic}, there is no penalty factor in \rev{the lower bound given in} \cref{cor:expldiff_ic}.
\rev{A}s a result\rev{, there is no environment where the lower bound is above 1,
which would have implied that all macroaction augmentations increase $p$-exploration difficulty}.
\rev{Nevertheless}, the answer \rev{to the question} is still affirmative.
The following two theorems \rev{construct} environments
where incorporating macroactions always increases $p$-exploration difficulty,
no matter how many there are or what they are.
\begin{theorem} \label{thm:expldiff-delta}
    Let $\M_{\base} = (S, A_{\base}, T_{\base}, g)$ be a solution-separable \sdmdp\ with a finite action space
    such that any state that has a length-1 solution only has length-1 solutions.
    Let $p$ be a probability distribution over solvable states.
    Suppose that $\delta > 0$ and
    \[
        \DKL{\rev{p}}{\rev{\rho_{\base,\delta}}} \leq \frac{\delta^2\rev{\log e}}{8(|A_{\base}| + 1)^2}\rev{.}\footnote{
        \rev{Technically, $\sum_{s} \rho_{\base,\delta}(s) \leq 1$ but may not equal 1,
        so the KL-divergence is really between $p'$ and $\rho_{\base,\delta}'$,
        defined such that $p' \equiv p$ and $\rho_{\base,\delta}' \equiv \rho_{\base,\delta}$
        on all solvable states and a dummy state $s_d$ is introduced
        that bears the remaining probability
        (i.e., $p'(s_d) = 0, \rho_{\base,\delta}'(s_d) = 1 - \sum_{s \neq s_d} \rho_{\base,\delta}(s)$).}
        }
    \]
    Then
    \(
        \expldiff(\M_{\aug}; p, \delta) > \expldiff(\M_{\base}; p, \delta)
    \)
    for any strict macroaction augmentation $\M_{\aug}$ of $\M_{\base}$.
\end{theorem}
\begin{theorem} \label{thm:expldiff_simp}
    Let $\M_{\base} = (S, A_{\base}, T_{\base}, g)$ be a solution-separable \sdmdp\ with a finite action space
    such that: 1) every action sequence is the solution to some state;
    2) for every \rev{solvable} state, \rev{all} solutions to that state have the same length.
    Let $p$ be a probability distribution over solvable states such that
    $p(s)/p(s') = |\Sol_{\base}(s)|/|\Sol_{\base}(s')|$ for any $s, s'$ whose solutions have the same length.
    Then
    \begin{multline}
        \expldiff(\M_{\aug}; p, 0) - \expldiff(\M_{\base}; p, 0) \\
        \geq \frac{|A_{\base}|}{|A_{\aug}|}\left(1 - \frac{|A_{\base}|}{|A_{\aug}|}\right)
    \end{multline}
    for any strict $A_{\aug}$-macroaction augmentation $\M_{\aug}$ of $\M_{\base}$.
    
    A stronger version of this theorem (\cref{app:expldiff_strong})
    relaxes the conditions on $\M_{\base}$ and $p$ and the modified bound
    involves subtracting a corresponding KL-divergence term.
\end{theorem}
Stated in words, \cref{thm:expldiff-delta} says that macroactions harm exploration
when most action sequences are solutions to some state
and that a state's \rev{assigned} importance $p(s)$ \rev{is} close to the probability that a uniformly random
\rev{action} sequence solves it. \Cref{thm:expldiff_simp} suggests that it suffices
for $p(s)$ to be roughly proportional to this probability across 
states whose solutions have the same length.
These results make more precise our intuition that it is more difficult to use
skills to improve exploration in environments where solutions to states
look uniformly randomly distributed.

\section{Experiments} \label{sec:expt_incompress}

\cref{cor:learndiff_ic,cor:expldiff_ic} suggest that solution-separable \sdmdp s with lower
\rev{unmerged} $p$-incompressibility can benefit more from macroactions.
We test this prediction on the four environments studied in \cref{sec:expt_diff},
which include both solution-separable (\texttt{RubiksCube222})
and non-solution-separable (\texttt{CliffWalking}, \texttt{CompILE2}, \texttt{8Puzzle}) \sdmdp s.
For different complexity measures $C$ ($p$-learning difficulty, $p$-exploration difficulty,
and sample complexity $N$ of four RL algorithms),
\Cref{fig:rl_improv-incompress} shows the best complexity improvement ratio \rev{$\min C_{\aug} / C_{\base}$}
\rev{across} the 31 (strict) macroaction augmentations of each base environment
against the unmerged $p$-incompressibility \rev{of the base environment}.
We \rev{observe} a positive correlation \rev{regardless of the choice of $C$ and RL algorithm},
thus corroborat\rev{ing} our theoretical predictions:
macroactions are more helpful in environments with lower \rev{unmerged} $p$-incompressibility.

{\revblock
While the definition of unmerged $p$-incompressibility is motivated in the context of macroactions
(\cref{cor:learndiff_ic,cor:expldiff_ic}),
experiments with general stochastic options discovered by LOVE \citep{jiang2022love}
show that it successfully captures the difficulty of applying HRL
with general options in an environment.
\Cref{tab:incompress_love} shows the unmerged $p$-incompressibility values of
our four environments, along with the sample complexity improvement ratio $N_{\aug} / N_{\base}$
\rev{from} optionally applying HRL with options discovered by LOVE.
The improvement from HRL decreases as the \rev{unmerged} $p$-incompressibility increases.


\section{$p$-Incompressibility for Skill Learning}
\label{sec:skill_learning}

\cref{app:skill_learning} demonstrates two ways to use our incompressibility measures
to derive objectives for skill learning.
We show that, under mild approximations, these two objectives are equivalent to two minimum description length (MDL)
objectives previously used in the skill learning literature.
In particular, finding the $A_{\aug}$ that minimizes $A_{\aug}$-merged $p$-incompressibility
corresponds to the objective used by LOVE \citep{jiang2022love},
and finding the skills such that the resultant skill-augmented environment has the highest unmerged $p$-incompressibility
corresponds to the objective used by LEMMA \citep{li2022lemma}.
}

\begin{table}[t]
\revblock
\caption{Unmerged $p$-incompressibility $\incompress(\M_{\base}; p)$
vs.\ the improvement ratio $N_{\aug} / N_{\base}$ of sample complexity $N_{r \geq r^*}$
from applying HRL with LOVE options.
Results are averaged over 5 seeds.
Because HRL can fail to learn an environment on some seeds,
we set the improvement ratio to 1 if HRL does not improve the sample complexity.
}
\label{tab:incompress_love}
\vskip 0.15in
\begin{center}
\begin{small}
\begin{tabular}{rll}
\toprule
Environment & $N_{\aug} / N_{\base}$ & $\incompress(\M_{\base}; p)$ \\
\toprule
\texttt{CliffWalking} & 0.000007 $\pm$ 0.000007 & 0.0000 \\
\texttt{CompILE2} & 0.00023 $\pm$ 0.00011 & 0.1475 \\
\texttt{8Puzzle} & 0.64 $\pm$ 0.19 & 0.5157 \\
\texttt{RubiksCube222} & 0.73 $\pm$ 0.17 & 0.8072 \\
\bottomrule
\end{tabular}
\end{small}
\end{center}
\vskip -0.1in
\end{table}

\section{Conclusion}

We introduce the first theoretical analysis of the utility of \rev{RL skills},
focusing on deterministic sparse-reward MDPs.
With both theoretical motivation and empirical verification,
we introduce metrics that quantify \rev{two aspects of RL complexity: exploration and learning from experience}.
We show both theoretically and experimentally that these metrics can be improved more
in environments where solutions to states are more compressible.
Further theoretical results suggest that skills benefit exploration more than learning from experience,
and that less expressive skills are less beneficial to improving RL sample efficiency.
Our work \rev{is a first step towards characterizing} the properties of an environment
that make skills helpful for RL,
and we expect future theoretical work to generalize beyond deterministic sparse-reward MDPs
\rev{with finite action spaces}.

{\revblock
\section*{Acknowledgements}

We thank anonymous referees for useful suggestions and discussions,
as well as instructors of the MIT Advanced Undergraduate Research Opportunities Program (SuperUROP)
for suggestions on presentation.

This work was funded by U.S.\ National Science Foundation (NSF) awards \#1918771 and \#1918839.
In addition, ZL was supported by the MIT Advanced Undergraduate Research Opportunities Program (SuperUROP)
and GP was supported by the Stanford Interdisciplinary Graduate Fellowship (SIGF).
}

\section*{Impact Statement}

This paper presents work whose goal is to advance the field of Machine Learning. There are many potential societal consequences of our work, none which we feel must be specifically highlighted here.

\bibliography{refs}
\bibliographystyle{icml2024}

\newpage
\appendix
\onecolumn

\section{Survey on Existing RL Difficulty Metrics} \label{app:rldiff}

Here, we provide a brief survey on existing RL difficulty metrics and
explain why they are inadequate for our purposes.
See \citet{conserva2022hardness} for a more detailed survey and benchmark.
We will be using the notation $\M = (S, A, P, R)$ for an MDP with
state space $S$, action space $A$, transition kernel $P$, and reward kernel $R$.
\begin{itemize}
    \item 
    The \textit{environmental value norm of the optimal policy} \citep{maillard2014valuenorm}
    is given by
    \begin{equation}
        \sup_{(s, a) \in S \times A} \sqrt{\var_{s' \sim P(s, a)} V_\gamma^*(s')},
    \end{equation}
    where $P(s, a)$ is the transition kernel of the MDP and $V_\gamma^*$
    is the value function of the optimal policy with discount factor $\gamma$.
    The variation in the values of next states quantifies the difficulty
    in obtaining accurate sample estimates of action values.
    However, in deterministic MDPs, which are our focus,
    the environmental value norm of the optimal policy is always zero and
    is therefore not applicable.
    \item 
    The \textit{distribution mismatch coefficient} \citep{kakade2002mismatchcoeff} is given by
    \begin{equation}
        \sup_\pi \sum_{s \in S} \frac{\mu_s^*}{\mu_s^\pi},
    \end{equation}
    where $\mu_s^\pi$ is the stationary distribution of the Markov chain
    induced by policy $\pi$ and $\mu_s^*$ is the stationary distribution of the
    Markov chain induced by the optimal policy.
    It measures how much the stationary distribution of states visited by the agent
    can differ from the optimal distribution.
    It is defined only for ergodic MDPs
    (otherwise the stationary distribution may not be uniquely defined)
    in the continuous setting,
    whereas we focus on deterministic MDPs
    (which are not ergodic when $|S| > 1$) in the episodic setting.
    \item
    The \textit{sum of reciprocals of suboptimality gaps} \citep{simchowitz2019suboptimalitygap}
    is given by
    \begin{equation}
        \sum_{(s, a) \in S \times A: \Delta(s, a) \neq 0} \frac{1}{\Delta(s, a)}, \quad
        \Delta(s, a) = V^*(s) - Q^*(s, a),
    \end{equation}
    where $V^*(s)$ and $Q^*(s, a)$ are the state and action value functions
    of the optimal policy.
    Larger $\Delta(s, a)$ allows the agent to more easily distinguish
    suboptimal actions from the optimal action and can thus reduce average total regret
    in the long run. However, as \citet{conserva2022hardness} points out,
    smaller $\Delta(s, a)$ makes it easier to find a near-optimal policy,
    which contributes to decreasing the sample complexity.
    \item
    The \textit{diameter} \citep{auer2008diameter} is defined to be
    \begin{equation}
        \sup_{s_1 \neq s_2} \inf_\pi T^{\pi}_{s_1 \to s_2},
    \end{equation}
    where $T^{\pi}_{s_1 \to s_2}$ denotes the expected time to reach $s_2$
    starting in $s_1$ following policy $\pi$.
    While this is defined for the continuous setting, a natural definition
    for the diameter of a \sdmdp\ $\M$ in the episodic setting would be
    \begin{equation}
        \sup_{s \neq g:\ \Sol_\M(s) \neq \emptyset} d_\M(s),
    \end{equation}
    where $d_\M(s)$ denotes the length of a shortest solution to $s$.
    However, taking the supremum is overly pessimistic, and in many cases,
    there may be states that are far from the goal but that we do not care about solving.
    Our $p$-learning difficulty takes this into account by using a weighted average of $d_\M(s)$,
    multiplied by $|A|$ to take into account the additional sample complexity
    due to a large action space.
\end{itemize}

\section{Environments} \label{app:envs}

Experiments were conducted on 4 base environments of varying complexity:
\begin{itemize}
    \item \texttt{CliffWalking} \citep{sutton2018reinforcement}, a toy grid world environment
    of size $4 \times 12$ where the agent always begins in the bottom left corner
    and has to travel to the bottom right corner.
    The available actions are moving one step in each of the 4 cardinal directions.
    The agent returns to its original position
    whenever it touches a square in the bottom row other than the leftmost and rightmost squares.
    \item \texttt{CompILE2} is one of the CompILE grid world environments
    \citep{kipf2019compile}. The agent navigates in an $10 \times 10$ grid world
    with walls both lining the edges and within the grid. The world also has
    several objects of different kinds, possibly with several of each kind.
    The agent's goal is to pick up several specified (kinds of) objects in order.
    In \texttt{CompILE2}, the agent has to pick up 2 objects.
    The available actions are moving one step in each of the 4 cardinal directions
    in addition to attempting to pick up the object in the current cell.
    The positions and types of the objects are fixed
    but the agent's position is randomized at every reset, following \citet{jiang2022love}.
    We did not choose 3 or more objects
    for the agent to pick up because we found that the agent could not find the positive reward
    signal without suitable skills in these cases,
    consistent with previous findings on the same environment \citep{kipf2019compile,jiang2022love}.
    Since whether the goal is reached depends on the sequence of objects
    the agent has picked up, the state includes both the grid and the sequence
    of objects that the agent has picked up thus far.
    Since \citet{kipf2019compile} did not publish the source code for the environment,
    we use the implementation by \citet{jiang2022love}.

    Because there can be several of the same kind of object on the grid,
    there are different sequences of objects the agent can pick up
    that amount to the same sequence of kinds of objects.
    There are thus multiple goal states, which are merged into one
    to comply with the definition of a \sdmdp.
    
    \item \texttt{8Puzzle} is the 8-puzzle, the $3\times3$ version of the more well-known 15-puzzle.
    There are 8 tiles numbered 1 to 8 on a $3\times3$ board so that there is one tile missing.
    The available actions are moving the position of the missing tile in each of the four cardinal directions.
    The solved state has the numbers 1 to 8 in order from left-to-right, top-to-bottom.
    The puzzle is scrambled from the solved state by applying a random legal action $K$
    times where $K$ is uniform between 1 and 31.
    Here, 31 is the maximum distance from any state to the goal state.
    The puzzle is re-scrambled if the scramble solves the cube.
    \item \texttt{RubiksCube222} is the 2x2 Rubik's cube, also called the pocket cube.
    The available actions are turning the front, right, or top faces clockwise by $90^\circ$.
    The cube is scrambled by applying a random sequence of moves of length $K$ where
    $K$ is uniform between 1 and 11 and where each move is turning the front, right, or top face
    $90^\circ$ clockwise, $180^\circ$, or $90^\circ$ counterclockwise and no two consecutive moves
    turn the same face. (Note that the action space used for scrambling is larger than
    the action space of the agent.) Here, 11 is the maximum number distance from a state to the solved state.
    We use the implementation provided by \citet{hukmani2021rubik}.
\end{itemize}
For \texttt{8Puzzle} and \texttt{RubiksCube222},
our choice of sampling the scramble length uniformly from 1 to some maximum $K$
follows \citet{agostinelli2019cube}.

Basic information about the 4 base environments is summarized in \cref{tab:envs}.
\begin{table}[b]
\caption{Basic information about the base environments studied by our experiments.
$|A_{\base}|$: size of base action space;
$|S|$: size of state space;
$|S_{p > 0}|$: size of support of $p$.
}
\label{tab:envs}
\vskip 0.15in
\begin{center}
\begin{small}
\begin{tabular}{rcccc}
\toprule
Environment & $|A_{\base}|$ & $|S|$ & $|S_{p > 0}|$ \\
\midrule
\texttt{CliffWalking} & 4 & 32 & 1 \\
\texttt{CompILE2} & 5 & 115,462 & 59 \\
\texttt{8Puzzle} & 4 & 362,880 & 181,439 \\
\texttt{RubiksCube222} & 3 & 3,674,160 & 3,674,159 \\
\bottomrule
\end{tabular}
\end{small}
\end{center}
\vskip -0.1in
\end{table}

For each base environment, one of the 32 action space variants is just the base environment itself.
The remaining 31 are (strict) macroaction augmentations generated as follows:
\begin{itemize}
    \item For \texttt{CliffWalking}, the LEMMA abstraction algorithm \citep{li2022lemma}
    found one single macroaction from the offline trajectory data generated using breadth-first search (BFS).
    That single macroaction is just the shortest sequence of actions that solves the only possible
    starting state of the environment: (U = up, R = right, D = down, L = left)
    \begin{itemize}
        \item URRRRRRRRRRRD
    \end{itemize}
    5 other sets of macroactions were derived from subsequences of near-optimal solutions to the starting state:
    \begin{itemize}
        \item RR
        \item RR, RRRR, RRRRRRRR
        \item RRRRRRRRRRR
        \item UUURRRR, RRR, DRDRD
        \item URRRRRRRRRRR, RRRRRRRRRRRD
    \end{itemize}
    Furthermore, for each $k = 1, 2, 3, 4, 5$, we randomly generated 5 sets of $k$ distinct macroactions.
    A random macroaction with length $L + 1$ ($L \sim \Geomtric(1/3)$) was generated as follows: 
    \begin{itemize}
        \item With probability 0.4, randomly choose between U and R with probabilities 0.3 and 0.7;
        \item With probability 0.3, randomly choose between R and D with probabilities 0.7 and 0.3;
        \item With probability 0.1, randomly choose between D and L with probabilities 0.7 and 0.3;
        \item With probability 0.2, randomly choose between L and U with probabilities 0.3 and 0.7.
    \end{itemize}
    We didn't choose probabilities uniform across all directions because this
    results in several sets of macroactions that cause the agent to drift leftward or downward
    during random exploration, and the agent almost never receives any positive reward signal.
    However, it was also the presence of drift that helped us generate variety in
    the learnability of the macroaction-augmented environments.
    Variation in the direction of the drift across different sets of macroactions
    resulted in sample efficiencies that varied across 7 orders of magnitude.
    
    \item For \texttt{CompILE2}, LEMMA discovered the following set of macroactions:
    (L = left, U = up, R = right, D = down, P = pick up)
    \begin{itemize}
        \item PUURRRP, LL, UU, DD
    \end{itemize}
    5 other sets of macroactions were derived from subsequences of subsets of these macroactions:
    \begin{itemize}
        \item LL, UU, DD
        \item LL, UU, RRR, DD
        \item PUU, RRRP
        \item PUURRRP
        \item PUURRRP, LL, UU, RRR, DD
    \end{itemize}
    Furthermore, for each $k = 1, 2, 3, 4, 5$, we randomly generated 5 sets of $k$ distinct macroactions.
    A random macroaction with length $L + 1$ ($L \sim \Geomtric(1/3)$) was generated as follows: 
    \begin{itemize}
        \item With probability 1/4, randomly choose among L, U and P with probabilities 0.4, 0.4 and 0.2;
        \item With probability 1/4, randomly choose among U, R and P with probabilities 0.4, 0.4 and 0.2;
        \item With probability 1/4, randomly choose among R, D and P with probabilities 0.4, 0.4 and 0.2;
        \item With probability 1/4, randomly choose among D, L and P with probabilities 0.4, 0.4 and 0.2.
    \end{itemize}
    
    \item For \texttt{8Puzzle}, LEMMA discovered the following set of macroactions:
    (U = up, R = right, D = down, L = left)
    \begin{itemize}
        \item RD, LDR
    \end{itemize}
    5 other sets of macroactions were derived from subsets of these macroactions,
    possibly with reflection across the diagonal (a symmetry of the puzzle):
    \begin{itemize}
        \item RD
        \item LDR
        \item RD, DR
        \item LDR, URD
        \item RD, DR, LDR, URD
    \end{itemize}
    Furthermore, for each $k = 1, 2, 3, 4, 5$, we randomly generated 5 sets of $k$ distinct macroactions.
    A random macroaction with length $L + 1$ ($L \sim \Geomtric(1/2)$) was generated
    by sampling from U, R, D, L with probabilities 0.2, 0.3, 0.3, 0.2.
    The higher probabilities for R and D are intended to encourage
    moving the position of the missing tile towards the bottom-right corner.
    
    \item For \texttt{RubiksCube222}, LEMMA generated the empty set. However, the 3 top-scoring macroactions were:
    (F = front face $90^\circ$, R = right face $90^\circ$, U = top face $90^\circ$)
    \begin{itemize}
        \item FF, RR, UU
    \end{itemize}
    5 other sets of macroactions were derived from subsets of these macroactions,
    possibly with more repetition of some base action:
    \begin{itemize}
        \item FF
        \item FF, FFF
        \item FF, RR
        \item FF, FFF, RR, RRR
        \item FF, FFF, RR, RRR, UUU
    \end{itemize}
    Note that FF, RR, UU are half-turns of faces (denoted F2, R2, U2 in standard cube notation)
    and FFF, RRR, UUU are counter-clockwise $90^\circ$ turns (usually denoted F$'$, R$'$, U$'$).
    
    Furthermore, for each $k = 1, 2, 3, 4, 5$, we randomly generated 5 sets of $k$ distinct macroactions.
    A random macroaction with length $L + 1$ ($L \sim \Geomtric(1/2)$) was generated
    by sampling from F, R, U each with probability 1/3.
\end{itemize}

\section{Experimental Details} \label{app:expt}

\subsection{Hyperparameters}

\begin{itemize}
    \item The learning rate is $\alpha = 0.1$ for
    Q-learning, value iteration and REINFORCE, and $\alpha = 0.0005$ for DQN.
    \item For the off-policy RL algorithms (Q-learning, value iteration, and DQN),
    the optimal epsilon schedule for epsilon greedy can vary by orders of magnitude
    across different action space variants of the same base environment.
    We therefore adopt an adaptive epsilon-greedy exploration policy where the probability $\eps$
    of choosing a random action starts at $1$ and is decreased by $0.002$ every time the agent beats
    its highest test reward so far by $0.002$, until $\eps = 0.1$.
    \item Testing was performed with $200$ episodes ($1$ episode for \texttt{CliffWalking},
    which only has one starting state) using the greedy policy
    (Q-learning, value iteration, DQN) or the current policy (REINFORCE).
    For the purposes of computing sample complexity, the $N$ at which a reward or value error threshold
    is reached is computed by averaging over all values of $N$ where
    the reward/value error crosses above/below the threshold.
    \item Experiments were run with a maximum of 100M environment steps.
    We applied early stopping with a test reward threshold of 0.95
    (0.75 for \texttt{RubiksCube222}) and average value error threshold of 0.025
    (0.1 for \texttt{RubiksCube222}).
    \item The horizon is $50$ for all environments, including skill-augmented environments.
    In addition, to simulate a cost of applying too many base actions,
    we terminate an episode whenever the number of base actions reaches $100$.
    \item For Q-learning, value iteration, and DQN, the replay buffer size is $1000$
    and updates are performed once every $4$ episodes with a batch size of $32$.
    \item Details on the model architecture of DQN are given in \cref{app:algo-specific}.
\end{itemize}
No extensive hyperparameter tuning was done as the purpose
of our experiments was not to compare RL algorithms,
but to compare the performance of one algorithm on different action space variants
of the same base environment.

\subsection{Computational Resources}

Experiments were run on 28 NVIDIA GPUs
($8 \times$Quadro RTX 5000, $8 \times$GeForce GTX 1080 Ti, $8 \times$Tesla V100 SXM2 32GB,
$4 \times$RTX 6000 Ada Generation).
One experiment, which usually consisted of 32 runs of some RL algorithm on different macroaction
augmentations of the same base environment, took between under a minute to about a week to finish.
In total, all experiments were completed within one month.

\subsection{Algorithm-Specific Details} \label{app:algo-specific}

\begin{itemize}
    \item Value iteration is modified to the RL setting in a way similar to
    Deep Approximate Value Iteration (DAVI) \citep{agostinelli2019cube}.
    In DAVI, a state is chosen from some initial distribution and the value network is updated
    by minimizing the quadratic loss between the current state value
    and the Bellman update, thus requiring a forward pass that computes the 
    values of all next states.
    In our version of value iteration, a state is chosen from the initial distribution
    and we apply a rollout of the epsilon-greedy policy.
    For each state $s$ in the rollout, we also compute all possible next states.
    Similar to Q-learning, these next states are stored along with $s$ in a replay buffer.
    When we sample a state (along with its next states) from the replay buffer,
    its value is updated in the direction of the Bellman update.
    Note that the fact that all possible next states are computed from each state in a rollout
    multiplies the number of environment steps taken by $|A|$.
    \item The policy $\pi_\theta(a \mid s)$ in REINFORCE is parameterized directly
    by the logits. In other words, the weights are an $|S| \times |A|$ matrix
    and $\pi_\theta(\cdot \mid s) = \mathrm{Softmax}(\theta_{s,\cdot})$.
    \item The implementation of the deep neural net in DQN depends on the environment.
    A state embedding is first constructed from the input
    before passed into a linear projection head that outputs the action values
    $Q(s, \cdot)$ of a state $s$.
    \begin{itemize}
        \item In \texttt{CliffWalking}, the input is a length-3 multihot vector
        at every location of the 4-by-12 grid (hence a $4 \times 12 \times 3$ binary tensor).
        In each multihot vector, the 3 indices represent the player, goal, and cliff.
        The state embedding is constructed by passing the input through
        a 2-layer CNN with ReLU activation followed by a 2-layer MLP
        with ReLU activation.
        The CNN has a kernel size of 3 and padding of 1.
        The hidden dimension is 32 and the output embedding has dimension 16.
        \item In \texttt{CompILE2}, the input has two components.
        The grid is represented as a length-12 multihot vector
        at every location of the 10-by-10 grid (hence a $10 \times 10 \times 12$ binary tensor).
        The 12 indices of each multihot vector represent the 10 types of objects, wall, and agent.
        The next object the agent has to pick up is represented as a length-10 one-hot vector.
        The grid is passed through a 2-layer CNN with ReLU activation followed by a 2-layer MLP with ReLU activation.
        The result is concatenated with an embedding of the next object the agent has to pick up
        and passed through a linear projection to form the final embedding of the observation.
        The CNN has a a kernel size of 3 and padding of 0.
        The hidden dimension is 32 in the CNN layers and 128 in the MLP layers;
        the object embedding has dimension 16; the output embedding has dimension 128.
        \item In \texttt{8Puzzle}, the input is a length-9 onehot vector
        at every location of the 3-by-3 grid (hence a $3 \times 3 \times 9$ binary tensor)
        denoting the tile present at each location (or the absence thereof).
        The state embedding is constructed by passing the input through
        a 2-layer CNN with ReLU activation followed by a 2-layer MLP
        with ReLU activation.
        The CNN has a kernel size of 3 and padding of 1.
        The hidden dimension is 32 and the output dimension is 32.
        \item In \texttt{RubiksCube222}, the input is a length-6 multihot vector
        for each of $6 \times 4 = 24$ tiles of the cube (hence a $24 \times 6$ binary tensor)
        denoting the color of each tile.
        The state embedding is constructed by flattening the input and
        passing it through a 4-layer MLP with ReLU activation.
        The hidden dimension is 64 and the output dimension is 32.
    \end{itemize}
\end{itemize}

\section{Additional Empirical Tests of $p$-Learning and $p$-Exploration Difficulty} \label{app:learndiff_extra}

\subsection{Empirically Verifying \Cref{lem:learndiff} for Motivating $p$-Learning Difficulty} \label{app:expt_learndiff}

To test how well $p$-learning difficulty captures learning from experience,
we study the value iteration algorithm for planning with known transitions and rewards in a \sdmdp.
We consider two variants of value iteration: state value iteration for learning
the values of states \citep{bellman1957valueiteration},
and action value iteration for learning the values of state-action pairs.
The latter is like Q-learning \citep{watkins1989qlearning} but
modified to update the values of all state-action pairs at once.
Instead of the original Bellman update, each update uses a linear interpolation
between the old value and the new value given by the Bellman update
with a learning rate of $\alpha = 0.1$ (see \cref{eq:value_iteration}).


For each base environment, we test the correlation between average solution length and sample complexity $N$
on 32 macroaction augmentations of that environment.
The results are summarized in \cref{tab:expt_results_no_expl}.
We find that the correlation between convergence time and average solution length is almost always greater
than 0.9, with it occasionally being near-perfect (above 0.99).
\begin{table*}[t]
\caption{Across 32 macroaction augmentations of each of 4 base environments,
we report the correlations between: the number of iterations until convergence ($N$)
for two variants of value iteration (state values and Q-values)
and two convergence criteria ($r \geq 0.9\rev{5}$; $\overline{\Delta V}$ or $\overline{\Delta Q} \leq 0.01$);
and the $p$-weighted mean solution length of a state ($\rev{\overline{d} :=\ } \mathbb E_{s\sim p}[d_{\aug}(s)]$).
The reported errors are standard errors of the \rev{mean over 5 seeds}.
}
\label{tab:expt_results_no_expl}
\vskip 0.15in
\begin{center}
\begin{small}
\begin{tabular}{rrcccc}
\toprule
&& \rev{$\overline{d}_\mathtt{CliffWalking}$} & \rev{$\overline{d}_\mathtt{CompILE2}$} & \rev{$\overline{d}_\mathtt{8Puzzle}$} & \rev{$\overline{d}_\mathtt{RubiksCube222}$} \\
\toprule
\multirow{2}{*}{Q-value iteration}
& $N_{r \geq 0.9\rev{5}}$
    & \rev{0.980 $\pm$ 0.001} & \rev{0.934 $\pm$ 0.012} & \rev{0.901 $\pm$ 0.013} & \rev{0.942 $\pm$ 0.007} \\
& $N_{\overline{\Delta Q} \leq 0.01}$
    & \rev{0.998 $\pm$ 0.000} & \rev{0.977 $\pm$ 0.003} & \rev{0.968 $\pm$ 0.006} & \rev{0.989 $\pm$ 0.001} \\
\midrule
\multirow{2}{*}{Value iteration}
& $N_{r \geq 0.9\rev{5}}$
    & \rev{0.977 $\pm$ 0.001} & \rev{0.942 $\pm$ 0.005} & \rev{0.902 $\pm$ 0.015} & \rev{0.942 $\pm$ 0.007} \\
& $N_{\overline{\Delta V} \leq 0.01}$
    & \rev{0.998 $\pm$ 0.000} & \rev{0.984 $\pm$ 0.002} & \rev{0.969 $\pm$ 0.005} & \rev{0.985 $\pm$ 0.001} \\
\bottomrule
\end{tabular}
\end{small}
\end{center}
\vskip -0.1in
\end{table*}

{\revblock
\subsection{Arithmetic Mean Variant of $p$-Exploration Difficulty Performs Worse Than the Geometric Mean} \label{app:expt_results_AM}

\Cref{tab:expt_results_AM} shows the version of \cref{tab:expt_results}
where $\expldiff = \log N_\text{GM}$ is redefined to be $\log N_\text{AM}$.
(Up to a constant factor, $N_\text{AM} = \mathbb E_{s \sim p}[1/q(s)]$ estimates an upper bound on the sample complexity
of the exploration stage of RL.)
Comparing the results with \cref{tab:expt_results},
we find that with 3 exceptions (in \texttt{8Puzzle}),
all correlation values are no higher than those when the geometric mean is used.\footnote{
The correlation values of \texttt{CliffWalking} are exactly equal across the two tables
because this environment has only one possible starting state, as a result of which
the arithmetic and geometric means are exactly equal.}
This provides empirical validation for using the geometric mean as opposed to the arithmetic mean
in our definition of $p$-exploration difficulty.
}

\begin{table*}[t]
\revblock
\caption{Version of \cref{tab:expt_results} where the geometric mean
is replaced with the arithmetic mean in the definition of $\expldiff$.
With 3 exceptions, all correlation values are no higher than those
when the geometric mean are used (\cref{tab:expt_results}).
}
\label{tab:expt_results_AM}
\vskip 0.15in
\begin{center}
\begin{small}
\begin{tabular}{rrcccc}
\toprule
&& $\log J_\mathtt{CliffWalking}$ & $\log J_\mathtt{CompILE2}$ & $\log J_\mathtt{8Puzzle}$ & $\log J_\mathtt{RubiksCube222}$ \\
\toprule
\multirow{2}{*}{Q-Learning}
& $\log N_{r \geq r^*}$
    & 0.947 $\pm$ 0.006 & 0.661 $\pm$ 0.049 & 0.301 $\pm$ 0.047 & 0.366 $\pm$ 0.081 \\
& $\log N_{\overline{\Delta Q} \leq \Delta Q^*}$
    & 0.953 $\pm$ 0.008 & 0.631 $\pm$ 0.061 & 0.442 $\pm$ 0.043 & 0.763 $\pm$ 0.019 \\
\midrule
\multirow{2}{*}{Value iteration}
& $\log N_{r \geq r^*}$
    & 0.933 $\pm$ 0.009 & 0.724 $\pm$ 0.043 & 0.788 $\pm$ 0.042 & 0.247 $\pm$ 0.058 \\
& $\log N_{\overline{\Delta V} \leq \Delta V^*}$
    & 0.951 $\pm$ 0.015 & 0.732 $\pm$ 0.035 & 0.877 $\pm$ 0.011 & 0.694 $\pm$ 0.021 \\
\midrule
\multirow{1}{*}{REINFORCE}
& $\log N_{r \geq r^*}$
    & 0.949 $\pm$ 0.006 & 0.732 $\pm$ 0.039 & 0.715 $\pm$ 0.020 & 0.537 $\pm$ 0.139 \\  
\midrule
\multirow{1}{*}{DQN}
& $\log N_{r \geq r^*}$
    & 0.789 $\pm$ 0.028 & 0.752 $\pm$ 0.075 & 0.621 $\pm$ 0.025 & 0.576 $\pm$ 0.023 \\
\bottomrule
\end{tabular}
\end{small}
\end{center}
\vskip -0.1in
\end{table*}

\section{Proofs} \label{app:proofs}

\begin{proof}[Proof of \cref{lem:learndiff}]
    (\textit{Note:} This proof assumes $\log$ refers to the natural logarithm.)
    
    For $\alpha = 1$, simple induction on $t$ shows that, at time $t$, the states with value $1$
    are exactly those states that can be solved with $t$ actions or less, and all other states have value $0$.
    
    For the $\alpha < 1$ case, let's first consider the case where
    the \sdmdp\ is a chain of states $0, 1, \ldots, n$ where state 0 is the only goal state
    and $T(s, a) = s - 1$ for any action $a$ and non-goal state $s \neq 0$.
    Then the value iteration formula becomes $V(s) \gets (1 - \alpha) V(s) + \alpha V(s-1)$ for $s > 0$
    and $V(0) = 1$.
    For $\alpha \ll 1$, we can write this as a differential equation
    \[
        \frac{dV_s}{dt} = -\alpha(V_s - V_{s-1})
    \]
    for $s > 0$, and $\frac{dV_0}{dt} = 0$.
    (We have switched to subscript notation to make it clearer that this is a linear system of ODEs in time.)
    Solving the system with the initial conditions $V_0(0) = 1$ and $V_s(0) = 0$ for $s > 0$ yields
    \[
        V_s(t) = 1 - e^{-\alpha t} \sum_{k=0}^{s-1} \frac{(\alpha t)^k}{k!}.
    \]
    Note that $V_s(t)$ decreases in $s$, i.e., at any time $t$, states closer to the goal have higher value.
    
    If $\alpha t = as + b\log(1/\eps)$ where $a > 1$ and $b = \frac{a}{a - 1}$, then
    \begin{align*}
        \log(1 - V_s(t)) &\leq -\alpha t + \log\left(s\frac{(\alpha t)^s}{s!}\right) \\
        &\leq -as - b\log(1/\eps) + s\log(as + b\log(1/\eps)) - (s-1)(\log(s-1) - 1) \\
        &\leq -as - b\log(1/\eps) + s\log s + s\log a + \frac{b}{a}\log(1/\eps) - (s-1)(\log(s-1) - 1) \\
        &= -s(a - \log a - 1) - \left(b - \frac{b}{a}\right)\log(1/\eps) + s(\log s - \log(s-1)) + \log(s-1) - 1 \\
        &= \log\eps - s\left(a - \log a - 1 + \log\left(\frac{s-1}{s}\right) - \frac{\log(s-1) - 1}{s}\right),
    \end{align*}
    which is less than $\log\eps$ for sufficiently large $s$ since $a - \log a - 1 > 0$.
    
    Let $\alpha t = s + \log(1/\eps) - 1 - \log 2$.
    Then for $s \geq 2$ and $\eps \leq 1/2$, we have
    \begin{align*}
        \log(1 - V_s(t)) &\geq -s - \log(1/\eps) + 1 + \log 2 + \log\left(\sum_{k=0}^{s-1} \frac{(s-1)^k}{k!}\right) \\
        &\overset{(*)}{\geq} -s - \log(1/\eps) + 1 + \log 2 + \log\left(\frac12 e^{s-1}\right) \\
        &= \log\eps,
    \end{align*}
    where the inequality marked (*) made use of the fact that the
    the median of a Poisson distribution with positive integer rate $s-1$ is exactly $s-1$ \citep{choi1994medianpoisson}.
    
    We have thus shown that we need $\alpha t = \Theta(s + \log(1/\eps))$ to obtain $1 - V_s(t) = \eps$.
    In other words, the time until the value estimate $V_s(t)$ is within $\eps$ of its true value of 1 is
    \begin{equation} \label{eq:VI_chain}
        t = \Theta\left(\frac{s + \log(1/\eps)}{\alpha}\right).
    \end{equation}
    
    Now let's return to the general graph setting. In this situation, the invariants are as follows:
    \begin{itemize}
        \item $\max_a V(T(s, a), t) = V(n(s), t)$ where $n(s)$ is the next state on the shortest path from $s$ to any goal.
        \item $V(s, t) = V_{d(s)}(t)$ where $V_d(t)$ is the solution to the value function in the case of a simple chain,
        as we just derived.
    \end{itemize}
    This invariants are preserved by the fact that $V_d(t)$ is non-increasing in $d$.
    Thus, replacing $s$ with $d(s)$ in the formula for the chain \sdmdp\ (\cref{eq:VI_chain})
    yields the result for the general \sdmdp\ case.
\end{proof}

\begin{proof}[Proof of \cref{thm:learndiff}]
    (\textit{Note:} We use the version of the geometric distribution with support excluding 0.)
    
    For $\sigma \in (A_{\aug})^+$,
    let $P_{\aug,\unif,\eps}(\sigma) = \eps(1 - \eps)^{|\sigma| - 1} |A_{\aug}|^{-|\sigma|}$,
    the probability that a random sequence of length $\sim \Geomtric(\eps)$
    with actions chosen uniformly from $A_{\aug}$ is exactly $\sigma$.
    Then $P_{\aug,\unif,\eps}$ is a probability distribution over $(A_{\aug})^+$, so
    \[
        \mathbb E_{\sigma \sim P_{\aug}}[-\log P_{\aug,\unif,\eps}(\sigma)]
        = \entropy[P_{\aug}, P_{\aug,\unif,\eps}]
        \geq \entropy[P_{\aug}],
    \]
    where $\entropy[p, q]$ denotes the cross entropy between $p$ and $q$.
    
    Now, fix any $0 < \eps < 1$. Then
    \begin{align*}
        \mathbb E_{s \sim p}[d_{\aug}(s)] \log\left(\frac{|A_{\aug}|}{1 - \eps}\right) + \log\left(\frac{1 - \eps}{\eps}\right)
        &= \mathbb E_{\sigma \sim P_{\aug}}[|\sigma|] \log\left(\frac{|A_{\aug}|}{1 - \eps}\right) + \log\left(\frac{1 - \eps}{\eps}\right) \\
        &= \mathbb E_{\sigma \sim P_{\aug}}[-\log P_{\aug,\unif,\eps}(\sigma)] \\
        &\geq \entropy[P_{\aug}].
    \end{align*}
    Thus,
    \begin{align*}
        \frac{\mathbb E_{s \sim p}[d_{\aug}(s)] \log\left(\frac{|A_{\aug}|}{1 - \eps}\right)}{{\mathbb E_{s \sim p}[d_{\base}(s)] \log\left(\frac{|A_{\base}|}{1 - \eps}\right)}}
        &\geq \frac{\entropy[P_{\aug}] - \log\left(\frac{1 - \eps}{\eps}\right)}{\mathbb E_{s \sim p}[d_{\base}(s)] \log\left(\frac{|A_{\base}|}{1 - \eps}\right)} \\
        \frac{\learndiff(\M_{\aug}; p)}{\learndiff(\M_{\base}; p)} = \frac{\mathbb E_{s \sim p}[d_{\aug}(s)]|A_{\aug}|}{\mathbb E_{s \sim p}[d_{\base}(s)]|A_{\base}|}
        &\geq \frac{|A_{\aug}|\log\left(\frac{|A_{\base}|}{1 - \eps}\right)}{|A_{\base}|\log\left(\frac{|A_{\aug}|}{1 - \eps}\right)} \frac{\entropy[P_{\aug}] - \log\left(\frac{1 - \eps}{\eps}\right)}{\mathbb E_{s \sim p}[d_{\base}(s)] \log\left(\frac{|A_{\base}|}{1 - \eps}\right)} \\
        &\geq \frac{|A_{\aug}|\log|A_{\base}|}{|A_{\base}|\log|A_{\aug}|} \frac{\entropy[P_{\aug}] - \log\left(\frac{1 - \eps}{\eps}\right)}{\mathbb E_{s \sim p}[d_{\base}(s)] \log\left(\frac{|A_{\base}|}{1 - \eps}\right)}.
    \end{align*}
    The last inequality used the fact that $|A_{\aug}| \geq |A_{\base}|$ gives
    \[
        \frac{\log\left(\frac{|A_{\base}|}{1 - \eps}\right)}{\log\left(\frac{|A_{\aug}|}{1 - \eps}\right)}
        = 1 - \frac{\log\left(\frac{|A_{\aug}|}{|A_{\base}|}\right)}{\log\left(\frac{|A_{\aug}|}{1 - \eps}\right)}
        \geq 1 - \frac{\log\left(\frac{|A_{\aug}|}{|A_{\base}|}\right)}{\log|A_{\aug}|}
        = \frac{\log|A_{\base}|}{\log|A_{\aug}|}.
    \]
    Now, we have
    \[
        \frac{\learndiff(\M_{\aug}; p)}{\learndiff(\M_{\base}; p)}
        \geq \sup_{0 < \eps < 1} \frac{|A_{\aug}|\log|A_{\base}|}{|A_{\base}|\log|A_{\aug}|} \frac{\entropy[P_{\aug}] - \log\left(\frac{1 - \eps}{\eps}\right)}{\mathbb E_{s \sim p}[d_{\base}(s)] \log\left(\frac{|A_{\base}|}{1 - \eps}\right)}
        = \frac{|A_{\aug}|\log|A_{\base}|}{|A_{\base}|\log|A_{\aug}|} \incompress_{A_{\aug}}(\M_{\base}; p),
    \]
    which completes the proof.
\end{proof}

\begin{proof}[Proof of \cref{cor:learndiff}]
    Since $|A_{\base}| \geq 2$, we have $|A_{\aug}| \geq |A_{\base}| + 1 \geq 3$. The function $f(x) = \ln x/x$ is decreasing for $x \geq e$,
    so
    \begin{multline*}
        \frac{|A_{\base}|\ln|A_{\aug}|}{|A_{\aug}|\ln|A_{\base}|}
        = \frac{f(|A_{\aug}|)}{f(|A_{\base}|)}
        \leq \frac{f(|A_{\base}| + 1)}{f(|A_{\base}|)}
        = \frac{|A_{\base}|\ln(|A_{\base}| + 1)}{(|A_{\base}| + 1)\ln|A_{\base}|} \\
        = \frac{|A_{\base}|}{|A_{\base}| + 1}\left(1 + \frac{\ln\left(1 + \frac{1}{|A_{\base}|}\right)}{\ln|A_{\base}|}\right)
        < \frac{|A_{\base}|}{|A_{\base}| + 1}\left(1 + \frac{1}{|A_{\base}|\ln|A_{\base}|}\right)
        = \frac{1}{|A_{\base}| + 1}\left(|A_{\base}| + \frac{1}{\ln|A_{\base}|}\right).
    \end{multline*}
    Then
    \[
        \frac{\learndiff(\M_{\aug}; p)}{\learndiff(\M_{\base}; p)} = \frac{|A_{\aug}|\ln|A_{\base}|}{|A_{\base}|\ln|A_{\aug}|} \incompress(\M_{\base}; p)
        > \frac{1 - \frac{1}{|A_{\base}| + 1}\left(1 - \frac{1}{\ln|A_{\base}|}\right)}{\frac{1}{|A_{\base}| + 1}\left(|A_{\base}| + \frac{1}{\ln|A_{\base}|}\right)}
        = 1,
    \]
    as desired.
\end{proof}

\begin{proof}[Proof of \cref{thm:expldiff_lb}]
    \begin{align}
        \expldiff(\M_{\aug}; p, \delta)
        &= \expect_{s \sim p}[-\log q_{\aug,\delta}(s)] \notag \\
        &= \expect_{s \sim p}[-\log \rho_{\aug,\delta}(s)] - \log\left(\frac{1 - \delta}{\delta}\right) \notag \\
        &= \expect_{s \sim p}\left[-\log\left(\frac{\rho_{\aug,\delta}(s)}{D(\M_{\aug}; \delta)}\right)\right] - \log\left(\frac{1 - \delta}{\delta}D(\M_{\aug}; \delta)\right) \notag \\
        &= \entropy[p] + \DKL{p}{\frac{\rho_{\aug,\delta}(\cdot)}{D(\M_{\aug}; \delta)}} - \log\left(\frac{1 - \delta}{\delta}D(\M_{\aug}; \delta)\right) \label{eq:expldiff_eq} \\
        &\geq \entropy[p] - \log\left(\frac{1 - \delta}{\delta}D(\M_{\aug}; \delta)\right), \notag
    \end{align}
    where we have used the fact that $\frac{\rho_{\aug,\delta}(\cdot)}{D(\M_{\aug}; \delta)}$ is a normalized probability distribution.

    Now, suppose the state space is finite and $\delta > \max_s p(s)$.
    According to \cref{eq:expldiff_eq},
    we want to show that we can make $\DKL{p}{\frac{\rho_{\aug,\delta}(\cdot)}{D(\M_{\aug}; \delta)}}$
    arbitrarily small with a suitable choice of $A_{\aug}$.
    Construct $A_{\aug}$ as follows. Let the number of skills $|A_{\aug}| - |A_{\base}|$
    be some large number $K \gg \max\{|A_{\base}|, 1 / \min_{s : p(s) > 0} p(s)\}$.
    For each solvable state $s$ with $p(s) > 0$, let $\lfloor Kf(s)\rfloor$
    skills send $s$ directly to the goal state and the remaining $K - \lfloor Kf(s)\rfloor$ send $s$ back to $s$ itself,
    where $f(s) = \frac{\delta}{\delta - (1 - \delta) p(s)} p(s) \in (0, 1)$.
    (For solvable states $s$ with $p(s) = 0$, simply let all $K$ skills send $s$ back to $s$ itself.)
    Let's now show that $\rho_{\aug,\delta}(s) \to p(s)$ as $K \to \infty$ for every solvable
    state $s$.

    $\rho_{\aug,\delta}(s)$ is the probability that an action sequence $\sigma$ with actions uniformly
    chosen from $A_{\aug}$ and length $|\sigma| \sim \Geomtric(\delta)$ solves $s$.
    Among all such action sequences, the total probability of those that have a base action
    is no more than the total probability of all actions sequences that have a base action.
    The latter is given by
    \begin{align*}
        1 - \sum_{\sigma \in (A_{\aug} \setminus A_{\base})^+} \delta(1 - \delta)^{|\sigma| - 1} |A_{\aug}|^{-|\sigma|}
        &= 1 - \sum_{l = 1}^\infty (|A_{\aug}| - |A_{\base}|)^l \delta(1 - \delta)^{l - 1} |A_{\aug}|^{-l} \\
        &= 1 - \frac{\delta}{1 - \delta} \sum_{l = 1}^\infty \left((1 - \delta)\left(1 - \frac{|A_{\base}|}{|A_{\aug}|}\right)\right)^l \\
        &= 1 - \frac{\delta \left(1 - \frac{|A_{\base}|}{|A_{\aug}|}\right)}{1 - (1 - \delta)\left(1 - \frac{|A_{\base}|}{|A_{\aug}|}\right)} \\
        &\to 0, \quad \text{as $|A_{\base}|/|A_{\aug}| \to 0$}.
    \end{align*}
    It now remains to show that the total probability of solutions to $s$ that consist only of skills
    approximates $p(s)$ arbitrarily well as $K \to \infty$.
    For $s$ with $p(s) = 0$, no such solutions exist and so their total probability is 0.
    For $s$ with $p(s) > 0$,
    \begin{align*}
        \sum_{\sigma \in \Sol_{\aug}(s) \cap (A_{\aug} \setminus A_{\base})^+} \delta(1 - \delta)^{|\sigma| - 1}|A_{\aug}|^{-|\sigma|}
        &= \sum_{l=1}^\infty \lfloor Kf(s)\rfloor (K - \lfloor Kf(s)\rfloor)^{l-1} \delta(1 - \delta)^{l - 1}|A_{\aug}|^{-l} \\
        &= \frac{\delta \lfloor Kf(s)\rfloor}{|A_{\aug}|} \sum_{l=1}^\infty \left((1 - \delta)\frac{K - \lfloor Kf(s)\rfloor}{|A_{\aug}|}\right)^{l-1} \\
        &= \frac{\delta \lfloor Kf(s)\rfloor}{|A_{\aug}|} \frac{1}{1 - (1 - \delta)\frac{K - \lfloor Kf(s)\rfloor}{|A_{\aug}|}} \\
        &\to \delta f(s) \frac{1}{1 - (1 - \delta)(1 - f(s))} \tag{as $K \to \infty$} \\
        &= p(s).
    \end{align*}
    By now, we have shown that $\rho_{\aug,\delta}(s) \to p(s)$ as $K \to \infty$ for every solvable state $s$.
    Since $S$ is finite, this convergence is uniform, so the KL-divergence between $p$
    and the normalized version of $\rho_{\aug,\delta}$ tends to zero as $K \to \infty$, as desired.
\end{proof}

\begin{proof}[Proof of \cref{cor:expldiff_ic}]
    Since $\M_{\base}$ is solution-separable and $\M_{\aug}$ is a macroaction augmentation of $\M_{\base}$,
    $\M_{\aug}$ is also solution-separable. Thus, $D(\M_{\aug}; \delta) \leq 1$.
    By \cref{thm:expldiff_lb},
    \[
        \expldiff(\M_{\aug}; p, \delta)
        \geq \entropy[p] - \log\left(\frac{1 - \delta}{\delta}D(\M_{\aug}; \delta)\right)
        \geq \entropy[p] - \log\left(\frac{1 - \delta}{\delta}\right),
    \]
    whereas
    \[
        \expldiff(\M_{\base}; p, \delta)
        = \expect_{s \sim p}[-\log q_{\base,\delta}(s)]
        \leq \expect_{s \sim p}\left[-\log\left(\left(\frac{1 - \delta}{|A_{\base}|}\right)^{d_{\base}(s)}\right)\right]
        = \expect_{s \sim p}[d_{\base}(s)]\log\left(\frac{|A_{\base}|}{1 - \delta}\right).
    \]
    Thus,
    \[
        \frac{\expldiff(\M_{\aug}; p, \delta)}{\expldiff(\M_{\base}; p, \delta)}
        \geq \frac{\entropy[p] - \log\left(\frac{1 - \delta}{\delta}\right)}{\expect_{s \sim p}[d_{\base}(s)]\log\left(\frac{|A_{\base}|}{1 - \delta}\right)},
    \]
    as desired.
\end{proof}

\begin{proof}[Proof of \cref{thm:incomp_expl}]
    The construction given in the proof of \cref{thm:expldiff_lb} allows us to
    make $\DKL{p}{\frac{\rho_{\aug,\delta}(s)}{D(\M_{\aug}; \delta)}}$ arbitrarily close to 0
    and $D(\M_{\aug}; \delta) = \sum_s \rho_{\aug,\delta}(s)$ arbitrarily close to 1
    with sufficient large $K = |A_{\aug}| - |A_{\base}|$.
    Recalling \cref{eq:expldiff_eq}, this means that for any $\eps > 0$,
    the construction gives
    \[
        \expldiff(\M_{\aug}; p, \delta) < \entropy[p] - \log\left(\frac{1 - \delta}{\delta}\right) + \eps
    \]
    for sufficiently large $K$.
    
    On the other hand,
    let $p', \rho_{\base,\delta}'$ be distributions defined on solvable states
    in addition to a dummy state $s_d$
    such that $p'(s) = p(s)$ and $\rho_{\base,\delta}'(s) = \rho_{\base,\delta}(s)$ whenever $s \neq s_d$,
    whereas $p'(s_d) = 0$
    and $\rho_{\base,\delta}'(s_d) = 1 - \sum_{s \neq s_d} \rho_{\base,\delta}(s)$.
    (Note that $\rho_{\base,\delta}'(s_d) \geq 0$ because $\M_{\base}$ is solution-separable.)
    Then $\DKL{p'}{\rho_{\base,\delta}'} > 0$ since $p' \not\equiv \rho_{\base,\delta}'$.
    This gives
    \begin{align}
        \expldiff(\M_{\base}; p, \delta)
        &= \expect_{s \sim p}\left[-\log\rho_{\base,\delta}(s)\right] - \log\left(\frac{1 - \delta}{\delta}\right) \notag \\
        &= \entropy[p] + \expect_{s \sim p}\left[-\log\frac{\rho_{\base,\delta}(s)}{p(s)}\right] - \log\left(\frac{1 - \delta}{\delta}\right) \notag \\
        &= \entropy[p] + \DKL{p'}{\rho_{\base,\delta}'(s)} - \log\left(\frac{1 - \delta}{\delta}\right) \label{eq:expldiff_eq2} \\
        &> \entropy[p] - \log\left(\frac{1 - \delta}{\delta}\right). \notag
    \end{align}
    As a result, for sufficiently large $K$,
    $\expldiff(\M_{\aug}; p, \delta) < \expldiff(\M_{\base}; p, \delta)$.

    Now, let's show that the construction in the proof of \cref{thm:expldiff_lb}
    can be made more precise to allow all states with $p(s) > 0$
    to have distinct canonical shortest solutions in $A_{\aug}$.
    Simply choose $K$ large enough so that, for all $s$ with $p(s) > 0$,
    the number of skills $\lfloor Kf(s)\rfloor$
    that send $s$ directly to $g$ is at least the number of states with $p(s) > 0$.
    Then the number of shortest solutions to every $s$ with $p(s) > 0$
    is at least the number of such $s$, so it is possible to choose one shortest
    solution for every such $s$ so that all the chosen solutions are distinct.
\end{proof}

\begin{proof}[Proof of \cref{cor:incomp_expl}]
    Define $A_{\aug}$ as in \cref{thm:incomp_expl},
    so that $\expldiff(\M_{\aug}; p, \delta) < \expldiff(\M_{\base}; p, \delta)$
    and \cref{thm:learndiff} gives
    \begin{equation}
        \frac{\learndiff(\M_{\aug}; p)}{\learndiff(\M_{\base}; p)}
        \geq \frac{|A_{\aug}|\log|A_{\base}|}{|A_{\base}|\log|A_{\aug}|}\incompress(\M_{\base}; p),
    \end{equation}
    which is identical to \cref{eq:learndiff_ic}.
    Then the proof that the additional condition in the corollary
    implies $\learndiff(\M_{\aug}; p) > \learndiff(\M_{\base}; p)$
    is identical to the proof of \cref{cor:learndiff}.
\end{proof}

\begin{proof}[Proof of \cref{thm:expldiff-delta}]
    Augment the state space of $\M_{\base}$ with a state $s_1$ that is solved by
    every length-1 sequence that is not already the solution to any other state.
    (Furthermore, all actions that do not result in the goal state instead transition to a dead state.)
    Denote by $\bar\M_{\base}$ the resultant \sdmdp\ and for simplicity of notation
    we write $\bar\rho_{\base,\delta}$ for $\rho_{\bar\M_{\base},\delta}$
    and $\bar d_{\base}$ for $d_{\bar\M_{\base}}$.
    Let $\bar\M_{\aug}$ denote the $A_{\aug}$-macroaction augmentation of $\bar\M_{\base}$.
    Then the solutions to $s_1$ in $A_{\aug}$ are exactly the same as those in $A_{\base}$
    since macroactions always have length greater than 1.
    We will write $\bar\rho_{\aug,\delta}$ to mean $\rho_{\bar\M_{\aug},\delta}$.
    Let $\bar p$ be a distribution over the solvable states of $\bar\M_{\base}$
    so that $\bar p = p$ on the solvable states of $\M_{\base}$ and $\bar p(s_1) = 0$.
    
    As in the proof of \cref{thm:incomp_expl}, we define distributions
    $p', \bar\rho_{\base,\delta}', \bar\rho_{\aug,\delta}'$ over the solvable states
    in addition to a dummy state $s_d$
    to be equal to $\bar p, \bar\rho_{\base,\delta}, \bar\rho_{\aug,\delta}$ whenever $s \neq s_d$,
    whereas $p'(s_d) = 0$ and $\bar\rho_{\base,\delta}'(s_d), \bar\rho_{\aug,\delta}'(s_d)$ are such that
    $\bar\rho_{\base,\delta}', \bar\rho_{\aug,\delta}'$ are normalized probability distributions.
    
    First, let's show that
    \begin{equation} \label{eq:AB-L1-dist}
        \sum_{s} |\bar\rho_{\aug,\delta}'(s) - \bar\rho_{\base,\delta}'(s)| \geq \frac{\delta}{|A_{\base}| + 1}.
    \end{equation}
    If $s$ is distance 1 away from the goal in $\bar\M_{\base}$, then
    \[
        \bar\rho_{\aug,\delta}'(s) = \delta\frac{n(s)}{|A_{\aug}|}, \quad \bar\rho_{\base,\delta}'(s) = \delta\frac{n(s)}{|A_{\base}|},
    \]
    where $n(s)$ denotes the number of solutions to $s$ in $\bar\M_{\base}$ (or equivalently, $\bar\M_{\aug}$),
    all of which have length 1.
    Thus,
    \[
        \sum_{s} |\bar\rho_{\aug,\delta}'(s) - \bar\rho_{\base,\delta}'(s)|
        \geq 
        \sum_{s: \bar d_{\base}(s) = 1} \delta n(s)\left(\frac{1}{|A_{\base}|} - \frac{1}{|A_{\aug}|}\right)
        = \delta|A_{\base}|\left(\frac{1}{|A_{\base}|} - \frac{1}{|A_{\aug}|}\right)
        = \delta\left(1 - \frac{|A_{\base}|}{|A_{\aug}|}\right)
        \geq \frac{\delta}{|A_{\base}| + 1},
    \]
    where the last inequality used the fact that $|A_{\aug}| \geq |A_{\base}| + 1$.
    
    We will now use \cref{eq:AB-L1-dist} to prove the theorem.
    By the triangle inequality,
    \[
        \sum_{s} |\bar\rho_{\aug,\delta}'(s) - \bar p'(s)|
        \geq \sum_{s} |\bar\rho_{\aug,\delta}'(s) - \bar\rho_{\base,\delta}'(s)|
        - \sum_{s} |\bar\rho_{\base,\delta}'(s) - \bar p'(s)|
        \geq \frac{\delta}{|A_{\base}| + 1} - \sum_{s} |\bar\rho_{\base,\delta}'(s) - \bar p'(s)|.
    \]
    Pinsker's inequality says that $\DKL{p}{q} \geq \frac12\left(\sum_x |p(x) - q(x)|\right)^2\rev{\log e}$
    for any two probability mass functions $p, q$. Thus, if
    \[
        \DKL{\bar p'}{\bar\rho_{\base,\delta}'} 
        = \DKL{p'}{\rho_{\base,\delta}'} < \frac{\delta^2\rev{\log e}}{8(|A_{\base}| + 1)^2},
    \]
    then
    \[
        \sum_{s} |\bar\rho_{\aug,\delta}'(s) - \bar p'(s)|
        > \frac{\delta}{|A_{\base}| + 1} - \sqrt{\frac{2}{\rev{\log e}} \cdot \frac{\delta^2\rev{\log e}}{8(|A_{\base}| + 1)^2}}
        = \frac{\delta}{2(|A_{\base}| + 1)}
    \]
    and so
    \[
        \DKL{p'}{\rho_{\aug,\delta}'}
        = \DKL{\bar p'}{\bar\rho_{\aug,\delta}'}
        > \frac12 \left(\frac{\delta}{2(|A_{\base}| + 1)}\right)^2\rev{\log e}
        > \DKL{p'}{\rho_{\base,\delta}'}.
    \]
    Now, by \cref{eq:expldiff_eq2}, this is equivalent to
    $\expldiff(\M_{\aug}; p, \delta) > \expldiff(\M_{\base}; p, \delta)$, as desired.
\end{proof}

The proof of \cref{thm:expldiff_simp} is omitted as the stronger version and its proof
are given in \cref{app:expldiff_strong}.

\section{Additional Theoretical Results}

{\revblock
\subsection{Preliminary Results on Stochastic Environments} \label{app:stochastic}

Here, we provide preliminary generalizations of our results for
stochastic sparse-reward MDPs, which are SDMDPs (\cref{def:sdmdp})
where the transition kernel $T$ may be stochastic (i.e., $T(s, a)$ is now a distribution over $S$).

In a (possibly stochastic) sparse-reward MDP,
let $W_{\sigma s}$ be the probability that taking actions $\sigma$ starting in $s$ results in the goal state.
For an ordering $\sigma_1, \sigma_2, \ldots$ of all positive-length action sequences in non-decreasing length,
define
\[
    w_{\sigma_k s} = \begin{cases}
        W_{\sigma_k s} & 1 \leq k < k_{max} \\
        1 - \sum_{i=1}^{k_{max}-1} W_{\sigma_i s} & k = k_{max} \\
        0 & k > k_{max}
    \end{cases}
\]
where $k_{max}$ is the largest $k$ such that $\sum_{i=1}^{k-1} W_{\sigma_i s} < 1$.
As a result, $\sum_\sigma w_{\sigma s} = 1$.

Let's redefine $d_{\mathcal M}(s)$ to be the weighted mean $\sum_\sigma w_{\sigma s}|\sigma|$,
so that $p$-learning difficulty (\cref{eq:learndiff})
and $A_{\aug}$-merged $p$-incompressibility (\cref{eq:incomp_A})
are now defined using this new notion of shortest solution length.
Furthermore, in the definition of $A_{\aug}$-merged $p$-incompressibility (\cref{eq:incomp_A}),
redefine $P_{\aug}$ to be $P_{\aug}(\sigma) = \sum_s p(s) w_{\sigma s}$ so that $\sum_\sigma P_{\aug}(\sigma) = 1$.
Note that the new definitions match the old definitions when the environment is deterministic.
The stochasticity effectively spreads the responsibility of being a ``shortest solution''
over several short solutions whose success probabilities $W_{\sigma s}$ add up to 1.

\begin{theorem}[Generalization of \cref{thm:learndiff}]
    Under the above redefinitions for stochastic sparse-reward MDPs,
    \cref{eq:learndiff_icA} of \cref{thm:learndiff} continues to hold.
\end{theorem}
\begin{proof}
    The proof is identical to that of the original \cref{thm:learndiff}.
\end{proof}

In stochastic environments, we can keep the original definition of $p$-exploration difficulty (\cref{eq:expldiff})
since the probabilistic definition of $q_{\M,\delta}(s)$ continues to make sense when there's stochasticity.
(As a reminder, it is the probability that
a uniformly random policy that terminates with probability $\delta$ before each step solves $s$.)
Similarly, we keep the definition of $\delta$-discounted solution density
(\cref{def:density}), which is also defined in terms of $q$.

\begin{theorem}[Generalization of the first half of \cref{thm:expldiff_lb}]
    Under the above redefinitions for stochastic sparse-reward MDPs,
    \cref{eq:expldiff_lb} of \cref{thm:expldiff_lb} continues to hold.
\end{theorem}
\begin{proof}
    The proof is identical to that of the original \cref{thm:expldiff_lb}.
\end{proof}
}

\subsection{Incorporating Skill Expressivity in \cref{thm:learndiff}} \label{app:express}

In \cref{thm:learndiff_express} below, we provide a version of \cref{thm:learndiff} that eliminates the dependence of
$\incompress_{A_{\aug}}(\M_{\base}; p)$ on $A_{\aug}$ and makes it depend explicitly on a quantitative measure of skill expressivity instead.
This new measure of incompressibility (\cref{eq:incomp_E}), which we call $E$-expressive $p$-incompressibility,
decreases in $E$. This is expected as an environment is more compressible
when the available skills are more expressive.

\begin{definition}[Quantifying skill expressivity]
    With respect to a \sdmdp\ $\M = (S, A, T, g)$,
    define the \textit{behavior variety expressivity} $E_z$ of a skill $z : S \to A^*$ to be $|z(S)|$,
    i.e., the number of distinct action sequences that $z$ can produce.
\end{definition}
\begin{definition}[$E$-expressive $p$-incompressibility]
    For a \sdmdp\ $\M = (S, A, T, g)$ with finite $|A| > 1$, define its \emph{$E$-expressive $p$-incompressibility} to be
    \begin{equation} \label{eq:incomp_E}
        \incompress(\M; p, E) = \sup_{0 < \varepsilon < 1} \frac{\min_{P} \mathrm H[P] - \log\left(\frac{1 - \varepsilon}{\varepsilon}\right)}{\mathbb E_{s\sim p}[d_{\M}(s)] \log\left(\frac{|A|E}{1 - \varepsilon}\right)}
    \end{equation}
    where the $\min_{P}$ is taken over all choices of
    canonical (not necessarily shortest) solutions to all states.\footnote{
    Recall that, given a choice of canonical solutions to all states,
    $P(\sigma)$ is the sum over $p(s)$ of all states $s$ that have $\sigma$ as their canonical solution.
    As a result, $\mathrm H[P] \leq \mathrm H[p]$ and equality holds in solution-separable \sdmdp s.}
    Note that expressivity $E$ occurs once in the denominator,
    so that larger $E$ results in smaller $\incompress(\M; p, E)$.
\end{definition}
\begin{theorem}[Expressivity and $p$-learning difficulty improvability] \label{thm:learndiff_express}
    Assuming the setup to \cref{thm:learndiff}, the following modified version of \cref{eq:learndiff_icA} holds:
    \begin{equation}
        \frac{J_{learn}(\M_{\aug}; p)}{J_{learn}(\M_{\base}; p)} \geq \frac{|A_{\aug}|\log|A_{\base}|}{|A_{\base}|\log|A_{\aug}|} \incompress(\M_{\base}; p, E)
    \end{equation} 
    where $E := \max_{z \in A_{\aug}\setminus A_{\base}} E_z$ is the
    maximum behavior variety expressivity of a skill in the skill augmentation.
    Higher expressivity $E$ thus reduces incompressibility
    and allows skills to improve $p$-learning difficulty more, as expected.
\end{theorem}
\begin{proof}
    Given any choice of canonical shortest solutions in $A_{\aug}$,
    define the random variables $\sigma_{\aug} \in (A_{\aug})^+$ and $\sigma_{\base} \in (A_{\base})^+$ as follows.
    For $s \sim p$, $\sigma_{\aug}$ is the canonical solution to $s$ in $A_{\aug}$,
    and $\sigma_{\base}$ is the same solution but with skills expanded into base actions.
    Then the distribution of $\sigma_{\aug}$ is just $P_{\aug}$,
    and let $P_{\base}$ be the distribution of $\sigma_{\base}$.

    Note that
    \begin{equation} \label{eq:express1}
        \mathrm{H}[P_{\aug}] + \mathrm{H}[\sigma_{\base} | \sigma_{\aug}] = \mathrm{H}[(\sigma_{\aug}, \sigma_{\base})] \geq \mathrm{H}[P_{\base}].
    \end{equation}
    Furthermore, since any $\sigma_{\aug}$ can expand to at most $E^{|\sigma_{\aug}|}$ different base action sequences,
    \begin{equation} \label{eq:express2}
        \mathrm{H}[\sigma_{\base} | \sigma_{\aug}] \leq \mathbb E_{\sigma_{\aug} \sim P_{\aug}}[|\sigma_{\aug}|\log E].
    \end{equation}
    In addition, recall from the proof of \cref{thm:learndiff} that, for any $0 < \eps < 1$,
    \begin{equation} \label{eq:express3}
        \mathbb E_{s \sim p}[d_{\aug}(s)] \log\left(\frac{|A_{\aug}|}{1 - \eps}\right) + \log\left(\frac{1 - \eps}{\eps}\right)
        \geq \mathrm{H}[P_{\aug}].
    \end{equation}
    Thus, substituting \cref{eq:express2,eq:express3} into \cref{eq:express1} yields
    \begin{align*}
        \mathbb E_{s \sim p}[d_{\aug}(s)] \log\left(\frac{|A_{\aug}|}{1 - \eps}\right) + \log\left(\frac{1 - \eps}{\eps}\right) + \mathbb E_{\sigma_{\aug} \sim P_{\aug}}[|\sigma_{\aug}|]\log E &\geq \mathrm{H}[P_{\base}] \\
        \mathbb E_{s \sim p}[d_{\aug}(s)] \log\left(\frac{|A_{\aug}|}{1 - \eps}\right) + \mathbb E_{s \sim p}[d_{\aug}(s)]\log E &\geq \mathrm{H}[P_{\base}] - \log\left(\frac{1 - \eps}{\eps}\right) \\
        \mathbb E_{s \sim p}[d_{\aug}(s)] \log\left(\frac{|A_{\aug}|E}{1 - \eps}\right) &\geq \mathrm{H}[P_{\base}] - \log\left(\frac{1 - \eps}{\eps}\right) \\
        \frac{\mathbb E_{s \sim p}[d_{\aug}(s)] \log\left(\frac{|A_{\aug}|E}{1 - \eps}\right)}{\mathbb E_{s \sim p}[d_{\base}(s)] \log\left(\frac{|A_{\base}|E}{1 - \eps}\right)} &\geq \frac{\min_{P_{\base}} \mathrm{H}[P_{\base}] - \log\left(\frac{1 - \eps}{\eps}\right)}{\mathbb E_{s \sim p}[d_{\base}(s)] \log\left(\frac{|A_{\base}|E}{1 - \eps}\right)} \\
        \frac{\learndiff(\M_{\aug}; p)}{\learndiff(\M_{\base}; p)}
        &\geq \frac{|A_{\aug}|\log\left(\frac{|A_{\base}|E}{1 - \eps}\right)}{|A_{\base}|\log\left(\frac{|A_{\aug}|E}{1 - \eps}\right)}
        \frac{\min_{P_{\base}} \mathrm{H}[P_{\base}] - \log\left(\frac{1 - \eps}{\eps}\right)}{\mathbb E_{s \sim p}[d_{\base}(s)] \log\left(\frac{|A_{\base}|E}{1 - \eps}\right)}
    \end{align*}
    where
    \[
        \frac{\log\left(\frac{|A_{\base}|E}{1 - \eps}\right)}{\log\left(\frac{|A_{\aug}|E}{1 - \eps}\right)} \geq \frac{\log|A_{\base}|}{\log|A_{\aug}|}
    \]
    since $\frac{E}{1 - \eps} > 1$ and $|A_{\aug}| \geq |A_{\base}|$. Thus,
    \[
        \frac{\learndiff(\M_A; p)}{\learndiff(\M_B; p)} \geq \frac{|A_{\aug}|\log|A_{\base}|}{|A_{\base}|\log|A_{\aug}|}
        \frac{\min_{P_{\base}} \mathrm{H}[P_{\base}] - \log\left(\frac{1 - \eps}{\eps}\right)}{\mathbb E_{s \sim p}[d_{\base}(s)] \log\left(\frac{|A_{\base}|E}{1 - \eps}\right)},
    \]
    which is true for all $0 < \eps < 1$, as desired.
\end{proof}

\subsection{Relaxing Solution-Separability Assumption in \cref{cor:learndiff_ic}}
\label{app:learndiff_ic_gen}

\begin{corollary}[Generalization of \cref{cor:learndiff_ic}] \label{cor:learndiff_ic_gen}
    Relaxing the solution-separability assumption, \cref{cor:learndiff_ic} holds if
    we replace $\mathrm H[p]$ in the definition of $\incompress(\M_0; p)$ with
    $\min_{P_{\base}} \entropy[P_{\base}]$. Here, $P_0$ is the distribution of canonical solutions
    to states sampled from $p$, and the minimum is taken over all possible choices of canonical solutions.
    Thus, $\entropy[P_{\base}]$ can be understood as the entropy
    of the state distribution if states with the same canonical solution are merged into one ``super-state.''
\end{corollary}
\begin{proof}
    The result follows directly from \cref{thm:learndiff_express} by setting $E = 1$.
\end{proof}

\subsection{Stronger Version of \cref{thm:expldiff_simp}} \label{app:expldiff_strong}

\textit{Note:} For notational simplicity, we will write $q_\M$ to mean $q_{\M,\delta=0}$.

Before stating the stronger version of \cref{thm:expldiff_simp},
we need to first define \textit{solution-length separations} of state spaces.
\begin{definition}
    For a \sdmdp\ $\M = (S, A, T, g)$, let $S_\solv$ denote the set of solvable states.
    The \textit{solution-length separation} $\tilde S_\solv$ of $S_\solv$
    is the result of separating every solvable state $s \in S_\solv$ into
    a set $\tilde S(s)$ of \textit{sub-states}
    corresponding to the lengths of solutions to $s$.
    Formally, we write
    \[
        \tilde S_\solv := \bigcup_{s \in S_\solv} \tilde S(s), \quad \tilde S(s) := \{(s, l) \mid \text{$l > 0$ s.t.\ $\exists \sigma \in \Sol_\M(s)$ with $|\sigma| = l$}\}.
    \]
    Furthermore, for a sub-state $\tilde s = (s, l)$ of $s$ corresponding to solution length $l$,
    we naturally define its solutions to be the length-$l$ solutions to $s$. Formally,
    \[
        \tilde\Sol_\M((s, l)) := \{\sigma \in \Sol_\M(s) \mid |\sigma| = l\}
    \]
    where the $\tilde{\phantom{\ }}$ is used to make it explicit that we're applying the operation to sub-states.
    
    Functions on $S_\solv$ defined using solutions to states can therefore be naturally extended to $\tilde S$.
    For example, $\tilde d(\tilde s)$ for $\tilde s = (s, l) \in \tilde S$ is just $l$, and
    \[
        \tilde q_\M(\tilde s) := \sum_{\sigma \in \tilde\Sol_\M(\tilde s)} |A|^{-|\sigma|} = \left|\tilde\Sol_\M(\tilde s)\right||A|^{-l} \quad \text{if $\tilde s = (s, l)$.}
    \]
    For an arbitrary function $f : S_\solv \to \mathbb R$, there is a family of natural extensions to $\tilde S_\solv$.
    Specifically, we say that $\tilde f : \tilde S_\solv \to \mathbb R$
    is a \textit{solution-length-separated additive extension} if, for all $s \in S_\solv$,
    \[
        f(s) = \sum_{\tilde s \in \tilde S(s)} \tilde f(\tilde s),
    \]
    and $f(s) = \tilde f(s)$ for $s \not\in S_\solv$.
    For example, $\tilde q_\M$ as defined above is a solution-length-separated additive extension to $q_\M$.
\end{definition}

\begin{theorem}[Generalization of \cref{thm:expldiff_simp}] \label{thm:expldiff}
    Let $\M_{\aug} = (S, A_{\aug}, T_{\aug}, g)$ be the $A_{\aug}$-macroaction augmentation of the solution-separable \sdmdp\ $\M_{\base} = (S, A_{\base}, T_{\base}, g)$
    with a finite action space, and $p$ a probability distribution over solvable states.
    Then there exists a solution-length-separated additive extension $\tilde p$ to $p$ in $\M_{\base}$ such that
    \begin{equation}
        \expldiff(\M_{\aug}; p) - \expldiff(\M_{\base}; p)
        \geq \frac{|A_{\base}|}{|A_{\aug}|} \left(1 - \frac{|A_{\base}|}{|A_{\aug}|}\right) - \DKL{\tilde p}{\tilde\lambda \tilde q_{\base}}.
    \end{equation}
    Here, $(\tilde\lambda\tilde q_{\base})((s, l)) := \tilde\lambda(l)\tilde q_{\base}((s, l))$, where
    \[
        \tilde\lambda(l) := \sum_{\tilde s': \tilde d_{\base}(\tilde s') = l} \tilde p(\tilde s')
    \]
    is the total probability (under $\tilde p$) of sub-states with solution length $l$
    and 
    \[
        \tilde q_{\base}((s, l)) = \left|\tilde\Sol_{\base}((s, l))\right||A_{\base}|^{-l}
    \]
    is the probability that a uniformly random action sequence of length $l$ is a solution to $s$.
    To make $\tilde\lambda \tilde q_{\base}$ a normalized probability distribution, we introduce a dummy sub-state
    $\tilde s_d(l)$ for each solution length $l$ with $\tilde p(\tilde s_d(l)) := 0$
    and $\tilde q_{\base}(\tilde s_d(l)) := 1 - \sum_{s: (s, l) \in \tilde S_\solv} \tilde q_{\base}((s, l))$.
    Note that $\tilde q_{\base}(\tilde s_d(l)) \geq 0$ because of solution-separability,
    and it is equal to zero when every action sequence of length $l$ is the solution to some state.
\end{theorem}

\begin{proof}[Proof of \cref{thm:expldiff}]
    For each solvable state $s$, denote by $\tilde S(s)$ the set of sub-states resultant from
    separating $s$ by solution length in the base environment.
    Define the solution-length-separated additive extension $\tilde p$ to $p$
    such that $\tilde p(\tilde s) \propto \tilde q_{\aug}(\tilde s)$
    for $\tilde s \in \tilde S(s)$, or more precisely,
    \[
        \tilde p(\tilde s) = p(s)\frac{\tilde q_{\aug}(\tilde s)}{q_{\aug}(s)}, \quad q_{\aug}(s) = \sum_{\tilde s \in \tilde S(s)} \tilde q_{\aug}(\tilde s).
    \]
    Here,
    \[
        \tilde q_{\aug}((s, l)) := \sum_{\substack{\sigma \in \Sol_{\aug}(s) \\ \text{$\sigma$ expands to $l$ base actions}}} |A_{\aug}|^{-|\sigma|}
    \]
    denotes the $q$ of the $A_{\aug}$-macroaction augmentation of $\tilde\M_{\base}$
    (the solution-length separation of $\M_{\base}$ with respect to $A_{\base}$), and not the solution-length separation of $\M_{\aug}$
    (the $A_{\aug}$-macroaction augmentation of $\M_{\base}$).

    Then
    \[
        \log\frac{q_{\base}(s)}{q_{\aug}(s)} = \log\sum_{\tilde s \in \tilde S(s)} \frac{\tilde q_{\aug}(\tilde s)}{q_{\aug}(s)} \frac{\tilde q_{\base}(\tilde s)}{\tilde q_{\aug}(\tilde s)} \geq \sum_{\tilde s \in \tilde S(s)} \frac{\tilde q_{\aug}(\tilde s)}{q_{\aug}(s)} \log \frac{\tilde q_{\base}(\tilde s)}{\tilde q_{\aug}(\tilde s)}
    \]
    by Jensen's inequality, so
    \[
        \expldiff(\M_{\aug}; p) - \expldiff(\M_{\base}; p) = \mathbb E_{s \sim p}\left[\log\frac{q_{\base}(s)}{q_{\aug}(s)}\right] \geq \mathbb E_{\tilde s \sim \tilde p}\left[\log\frac{\tilde q_{\base}(\tilde s)}{\tilde q_{\aug}(\tilde s)}\right].
    \]
    It thus suffices to lower-bound the latter.
    Let's consider base solution lengths $l$ separately.

    Fix some $l \geq 1$. Define $\tilde p_l$ to be the conditional distribution of $\tilde p$
    on sub-states with solution length $l$. In other words, if $\tilde S_l$ denotes
    the set of sub-states with solution length $l$, then $\tilde p_l$ is a distribution over $\tilde S_l$ defined as
    \[
        \tilde p_l(\tilde s) = \frac{\tilde p(\tilde s)}{\tilde\lambda(l)}, \quad
        \tilde\lambda(l) = \sum_{\tilde s' \in \tilde S_l} \tilde p(\tilde s').
    \]
    We write
    \begin{equation} \label{eq:sepexpl}
        \mathbb E_{\tilde s \sim \tilde p_l}\left[\log\frac{\tilde q_{\base}(\tilde s)}{\tilde q_{\aug}(\tilde s)}\right]
        = \mathbb E_{\tilde s \sim \tilde p_l}\left[\log\frac{\tilde q_{\base}(\tilde s)}{\tilde q_{\aug}(\tilde s) / \sum_{\tilde s' \in \tilde S_l^*} \tilde q_{\aug}(\tilde s')}\right]
        - \log\sum_{\tilde s' \in \tilde S_l^*} \tilde q_{\aug}(\tilde s'),
    \end{equation}
    where $\tilde S_l^*$ denotes the set $\tilde S_l$ of sub-states with solution length $l$, along with
    a dummy state $\tilde s_d$ for every length-$l$ action sequence that isn't a solution to any state.
    Note that $\tilde q_{\base}(\tilde s_d) = |A_{\base}|^{-l}$ for each dummy state
    so that $\sum_{\tilde s' \in \tilde S_l^*} \tilde q_{\base}(\tilde s') = |A_{\base}|^l|A_{\base}|^{-l} = 1$, whereas
    $\tilde q_{\aug}(\tilde s_d) = \sum_{\text{$\sigma \in (A_{\aug})^+$ expands to $\alpha$}} |A_{\aug}|^{-|\sigma|}$
    where $\alpha$ is the action sequence assigned as the solution to $\tilde s_d$. 
    As usual for dummy states, we define $\tilde p_l(\tilde s_d) = 0$.

    Let's first lower-bound the first term on the RHS of \cref{eq:sepexpl}:
    \begin{equation} \label{eq:bound-rhs1}
        \mathbb E_{\tilde s \sim \tilde p_l}\left[\log\frac{\tilde q_{\base}(\tilde s)}{\tilde q_{\aug}(\tilde s) / \sum_{\tilde s' \in \tilde S_l^*} \tilde q_{\aug}(\tilde s')}\right]
        = \DKL{\tilde p_l}{\frac{\tilde q_{\aug}(\cdot)}{\sum_{\tilde s' \in \tilde S_l^*} \tilde q_{\aug}(\tilde s')}}
        - \DKL{\tilde p_l}{\tilde q_{\base}}
        \geq -\DKL{\tilde p_l}{\tilde q_{\base}}.
    \end{equation}

    Let's now upper-bound the sum in the second term on the RHS of \cref{eq:sepexpl}.
    We write
    \begin{align*}
        \sum_{\tilde s' \in \tilde S_l^*} \tilde q_{\aug}(\tilde s')
        &= \sum_{\substack{\sigma \in (A_{\aug})^* \\ \text{$\sigma$ expands to $l$ base actions}}} |A_{\aug}|^{-|\sigma|}
    \end{align*}
    which is a function $f_l(x_2, \ldots, x_K)$
    where $x_k$ is the number of macroactions of length $k$ divided by $|A_{\aug}|$
    and $K$ is the maximum length of any macroaction.
    To see that this is a function of only $l$ and $x_k$,
    notice that changing the number of macroactions of every length as well as the number of base actions
    by the same factor $\xi$ (which keeps all $x_k$ unchanged)
    will result in $\xi^{l'}$ times more sequences $\sigma \in (A_{\aug})^*$ such that $|\sigma| = l'$
    and $\sigma$ expands to $l$ base actions,
    whereas the $|A_{\aug}|^{-l'}$ summand is multiplied by a factor of $\xi^{-l'}$ for these sequences.
    The two factors cancel out, thus leaving the entire sum unchanged.

    Now, let's derive a recursive formula for $f_l(x_2, \ldots, x_K)$ where the $x_i$ are treated as parameters.
    To do this, we separate the sum over $\sigma$ into cases depending on whether the first action in $\sigma$
    is a macroaction, and its length if yes.
    If the first action in $\sigma$ is a base action, then the rest of $\sigma$ expands to length $l-1$,
    so the contribution to the sum is $x_1 f_{l-1}(x_2, \ldots, x_K)$,
    where $x_1 := 1 - \sum_{k=2}^K x_k$ is the number of base actions divided by $|A_{\aug}|$.
    If the first action in $\sigma$ is a macroaction of length $k$, then the rest of $\sigma$ expands to length $l-k$,
    so the contribution to the sum is $x_k f_{l-k}(x_2, \ldots, x_K)$. To summarize,
    \[
        f_l = \sum_{k=1}^K x_k f_{l-k},
    \]
    where it is understood that $f_i = 0$ for $i < 0$. The base case is $f_0 = 1$.
    Since the sum of the coefficients $x_k$ in the recursive formula equals 1,
    $f_l$ is just a weighted average of $f_{l-1}, f_{l-2}, \ldots, f_{l-K}$.
    Thus, if $f_l \leq a$ for $1 \leq l \leq K$ then $f_l \leq a$ for all $l \geq 1$.

    Let's show by induction on $K$ that
    \begin{equation} \label{eq:recur-ineq}
        f_l \leq 1 - x_1 + x_1^2
    \end{equation}
    for all $l \geq 1$.
    It suffices to show that $f_l \leq 1 - x_1 + x_1^2$ for $1 \leq l \leq K$.

    For $K = 1$, $f_1 = x_1 = 1 = 1 - x_1 + x_1^2$.
    For $K = 2$, $f_1 = x_1 \leq x_1 + (1 - x_1)^2 = 1 - x_1 + x_1^2$
    and $f_2 = x_1 f_1 + (1 - x_1) = 1 - x_1 + x_1^2$.

    Now, for $K \geq 3$, assume the $K-1$ and $K-2$ cases hold.

    Let's upper-bound $f_l$ for the following two cases separately: (i) $1 \leq l \leq K-1$; (ii) $l = K$.
    
    (i) Define $x_k' = x_k / \bar x_K$ for $1 \leq k \leq K-1$ where $\bar x_K := 1 - x_K = \sum_{i=1}^{K-1} x_i$.
    Define the sequence $f'_l = \sum_{k=1}^{K-1} x_k' f'_{l-k}$ with $f'_i = 0$ for $i < 0$ and $f'_0 = 1$.
    Then the inductive hypothesis gives $f'_l \leq 1 - x_1' + x_1'^2$ for $1 \leq l \leq K-1$.
    It is easy to show by induction on $l$ that $f_l = (\bar x_K)^l f'_l$ for $0 \leq l \leq K-1$,
    so for $1 \leq l \leq K-1$,
    \[
        f_l \leq \bar x_K (1 - x_1' + x_1'^2) = \bar x_K - x_1 + \frac{x_1^2}{\bar x_K},
    \]
    where $\bar x_K$ is restricted to the range $x_1 \leq \bar x_K \leq 1$.
    Since $\bar x_K + \frac{x_1^2}{\bar x_K}$ is increasing for $\bar x_K \geq x_1$,
    its maximum is reached when $\bar x_K = 1$, i.e.,
    \[
        f_l \leq \bar x_K - x_1 + \frac{x_1^2}{\bar x_K} \leq 1 - x_1 + x_1^2.
    \]

    (ii) Note that recursively expanding the recursion formula for $f_l$ until we reach the base cases
    results in a polynomial in $x_1, \ldots, x_K$.
    It is easy to see by induction on $l$ that, for $1 \leq l \leq K$, no term contains $x_k$ where $k > l$
    and there is a single term containing $x_l$ which is just $x_l$.
    So $f_{K-1} = P_1(x_1, \ldots, x_{K-2}) + x_{K-1}$ for some polynomial $P_1$, and
    \begin{align*}
        f_K &= \sum_{k=1}^K x_k f_{K-k} \\
        &= x_1(P_1(x_1, \ldots, x_{K-2}) + x_{K-1}) + \sum_{k=2}^{K-2} x_k f_{K-k} + x_{K-1} f_1 + x_K \\
        &= P_2(x_1, \ldots, x_{K-2}) + 2 x_1 x_{K-1} + x_K
    \end{align*}
    for some polynomial $P_2$. Substituting $x_K = 1 - \sum_{k=1}^{K-1} x_k$ results in
    \[
        f_K = P_3(x_1, \ldots, x_{K-2}) + x_{K-1}(2 x_1 - 1)
    \]
    for some polynomial $P_3$. This is linear in $x_{K-1}$ where $0 \leq x_{K-1} \leq 1 - \sum_{i=1}^{K-2} x_i$, so
    \[
        f_K \leq \max\left\{f_K|_{x_{K-1} = 0}, f_K|_{x_K = 0}\right\}
    \]
    where
    \begin{align*}
        f_K|_{x_{K-1} = 0} &= P_3(x_1, \ldots, x_{K-2}) \\
        f_K|_{x_K = 0} &= P_3(x_1, \ldots, x_{K-2}) + \left(1 - \sum_{i=1}^{K-2} x_i\right)(2x_1 - 1).
    \end{align*}
    $f_K|_{x_K = 0} \leq 1 - x_1 + x_1^2$ by the inductive hypothesis. Now let's upper-bound $f_K|_{x_{K-1} = 0}$.
    
    Note that, regardless of the value of $x_{K-1}$, we have $f_0 = 1$ and $f_k \leq 1 - x_1 + x_1^2$ for $1 \leq k \leq K - 2$
    by the inductive hypothesis, since these values of $f_k$ are independent of $x_{K-1}$.
    Thus,
    \begin{align*}
        f_{K-1}|_{x_{K-1} = 0} &= \sum_{k=1}^K x_k f_{K-1-k} \leq \sum_{k=1}^{K-2} x_k (1 - x_1 + x_1^2) = (1 - x_K)(1 - x_1 + x_1^2) \\
        f_K|_{x_{K-1} = 0} &= \sum_{k=1}^K x_k f_{K-k} \\
        &\leq x_1(1 - x_K)\left(1 - x_1 + x_1^2\right) + \sum_{k=2}^{K-2} x_k\left(1 - x_1 + x_1^2\right) + x_K \\
        &= (x_1(1 - x_K) + 1 - x_1 - x_K)\left(1 - x_1 + x_1^2\right) + x_K \\
        &= (1 - x_K (1 + x_1))\left(1 - x_1 + x_1^2\right) + x_K \\
        &= 1 - x_1 + x_1^2 - x_K\left(1 + x_1^3\right) + x_K \\
        &\leq 1 - x_1 + x_1^2.
    \end{align*}
    Thus, we have shown that $f_K \leq 1 - x_1 + x_1^2$, which completes the inductive step.
    This concludes the proof of \cref{eq:recur-ineq}.

    Now,
    \begin{equation} \label{eq:bound-rhs2}
        -\log \sum_{\tilde s' \in \tilde S_l^*} \tilde q_{\aug}(\tilde s') = -\log f_l \geq 1 - f_l \geq \frac{|A_{\base}|}{|A_{\aug}|}\left(1 - \frac{|A_{\base}|}{|A_{\aug}|}\right),
    \end{equation}
    where the last inequality follows from \cref{eq:recur-ineq} with $x_1 = |A_{\base}|/|A_{\aug}|$.


    We now substitute \cref{eq:bound-rhs1,eq:bound-rhs2} into \cref{eq:sepexpl} to obtain
    \[
        \mathbb E_{\tilde s \sim \tilde p_l}\left[\log\frac{\tilde q_{\base}(\tilde s)}{\tilde q_{\aug}(\tilde s)}\right] \geq \frac{|A_{\base}|}{|A_{\aug}|} \left(1 - \frac{|A_{\base}|}{|A_{\aug}|}\right) - \DKL{\tilde p_l}{\tilde q_{\base}}.
    \]
    Thus, we finally have
    \begin{align*}
        \expldiff(\M_{\aug}; p) - \expldiff(\M_{\base}; p)
        &\geq \mathbb E_{\tilde s \sim \tilde p}\left[\log\frac{\tilde q_{\base}(\tilde s)}{\tilde q_{\aug}(\tilde s)}\right] \\
        &= \sum_{l=1}^\infty \tilde\lambda(l)\mathbb E_{\tilde s \sim \tilde p_l}\left[\log\frac{\tilde q_{\base}(\tilde s)}{\tilde q_{\aug}(\tilde s)}\right] \\
        &\geq \sum_{l=1}^\infty \tilde\lambda(l) \left(\frac{|A_{\base}|}{|A_{\aug}|} \left(1 - \frac{|A_{\base}|}{|A_{\aug}|}\right) - \DKL{\tilde p_l}{\tilde q_{\base}}\right) \\
        &\geq \frac{|A_{\base}|}{|A_{\aug}|} \left(1 - \frac{|A_{\base}|}{|A_{\aug}|}\right) - \DKL{\tilde p}{\tilde\lambda \tilde q_{\base}}.
    \end{align*}
\end{proof}

{\revblock
\section{Relating $p$-Incompressibility to Skill Learning} \label{app:skill_learning}

The intuition that skills should optimally compress successful trajectories
has been previously used by skill-discovery algorithms like LOVE and LEMMA.
Using the incompressibility measures introduced in this paper,
we can approach skill learning more rigorously.
There are two approaches to converting $p$-incompressibility into a skill-learning objective.

The first approach is to find $A_{\aug}$ that minimizes the lower bound
on the $p$-learning difficulty increase ratio as given in \cref{thm:learndiff}.
This is equivalent to minimizing
\begin{equation}
    \mathcal L_1(A_{\aug}) =
    \frac{|A_{\aug}|}{\log|A_{\aug}|}
    \sup_{0 < \eps < 1} \frac{\entropy[P_{\aug}] - \log\left(\frac{1 - \eps}{\eps}\right)}{\log\left(\frac{|A_{\base}|}{1 - \eps}\right)},
\end{equation}
where the $\sup$ factor is proportional to the $A_{\aug}$-merged $p$-incompressibility.
Usually, $\entropy[P_{\aug}]$ is large, as a result of which the maximizing $\eps$
satisfies $\eps \ll 1$ and $\entropy[P_{\aug}] \gg \log\left(\frac{1 - \eps}{\eps}\right)$.
Thus, minimizing $\mathcal L_1(A_{\aug})$ becomes equivalent to minimizing
\begin{equation}
    \mathcal L_2(A_{\aug}) = \frac{|A_{\aug}|}{\log|A_{\aug}|} \entropy[P_{\aug}].
\end{equation}
When $|A_{\aug}|$ is known or given as a hyperparameter, then the objective is to minimize
\begin{equation}
    \mathcal L_3(A_{\aug}) = \entropy[P_{\aug}].
\end{equation}
Note that in practice, it is not possible to compute $P_{\aug}$, the distribution of shortest solutions
using actions from $A_{\aug}$ to states generated by the environment.
However, we do have a training set of offline experience, so we can use our skills to rewrite these
solutions and define $\hat P_{\aug}$ to be the resultant empirical distribution of abstracted solutions.
The $P_{\aug}$ appearing in the objectives $\mathcal L_1, \mathcal L_2, \mathcal L_3$ should thus be interpreted as
$\hat P_{\aug}$ as calculated from our training set.

However, the resultant approximation of $\entropy[P_{\aug}]$
will be a significant under-approximation if the training set is much smaller than the
number of states that cover most of the state space under $p$.
In this case, we recommend modeling $P_{\aug}$ with the assumption that
it is generated by sampling i.i.d.\ actions from a distribution $p_{a,\aug}$ over $A_{\aug}$,
with solution length sampled from a distribution $p_{l,\aug}$.
Then the maximum-likelihood (ML) estimates of $p_{a,\aug}$ and $p_{l,\aug}$
are just the empirical distribution of actions and the empirical distribution of solution lengths
in the abstracted training set. If we define $\tilde P_{\aug}$ to be the distribution of action sequences
defined by this choice of $p_{a,\aug}$ and $p_{l,\aug}$, then we can approximate
$\entropy[P_{\aug}] \approx \entropy[\hat P_{\aug}, \tilde P_{\aug}]
= \entropy[p_{l,\aug}] + \overline l_{\aug} \entropy[p_{a,\aug}]$,
where $\overline l_{\aug} := \expect_{l \sim p_{l,\aug}}[l]$ is the average solution length.
Under this approximation, $\mathcal L_3$ becomes
\begin{equation}
    \mathcal L_4(A_{\aug}) = \entropy[p_{l,\aug}] + \overline l_{\aug} \entropy[p_{a,\aug}].
\end{equation}
(We can similarly apply this approximation to $\mathcal L_1$ and $\mathcal L_2$.)
It is often the case that $\entropy[p_{l,\aug}]$ is much smaller than $\overline l_{\aug} \entropy[p_{a,\aug}]$,
so neglecting that term results in the objective
\begin{equation}
    \mathcal L_5(A_{\aug}) = \overline l_{\aug} \entropy[p_{a,\aug}].
\end{equation}
Note that $\mathcal L_5$ is exactly the minimum description length (MDL) objective used by LOVE \citep{jiang2022love}.
It represents the average number of bits required to encode an abstracted solution,
where the encoding of actions is optimized for the empirical distribution of actions in the abstracted training set.

The second approach to deriving a skill learning objective from $p$-incompressibility
is based on the idea that the maximally abstracted environment is the least compressible.
Using unmerged $p$-incompressibility to measure incompressibility, this corresponds to the \textit{maximization} objective
\begin{equation}
    \mathcal J_6(A_{\aug}) = \incompress(\M_{\aug}; p) =
    \sup_{0 < \eps < 1} 
    \frac{\entropy[p] - \log\left(\frac{1 - \eps}{\eps}\right)}{\mathbb E_{s \sim p}[d_{\aug}(s)] \log\left(\frac{|A_{\aug}|}{1 - \eps}\right)}.
\end{equation}
Similar to $\entropy[P_{\aug}]$ in $\mathcal L_1$, $\entropy[p]$ in $\mathcal J_6$ is often large,
in which case the maximizing $\eps$ satisfies $\eps \ll 1$ and $\entropy[p] \gg \log\left(\frac{1 - \eps}{\eps}\right)$.
Under this approximation, the maximization objective becomes the minimization objective
\begin{equation}
    \mathcal L_7(A_{\aug}) = \expect_{s \sim p}[d_{\aug}(s)]\log|A_{\aug}|.
\end{equation}
As with $P_{\aug}$, $\expect_{s \sim p}[d_{\aug}(s)]$ cannot be computed exactly,
so we approximate it with the average solution length in the abstracted training set, i.e., $\overline l_{\aug}$.
As a result,
\begin{equation}
    \mathcal L_7(A_{\aug}) = \overline l_{\aug}\log|A_{\aug}|,
\end{equation}
which is just $\mathcal L_5$ but with a uniform distribution for $p_{a,\aug}$.
It can thus also be interpreted as an MDL objective where the encoding of actions is a uniform code.
Note that this is exactly the objective used by LEMMA \citep{li2022lemma}.
}


\end{document}